\documentclass[11pt]{article}
\usepackage[a4paper,margin=1in]{geometry}

\usepackage{moreverb,url}
\usepackage{natbib}
\usepackage[colorlinks,bookmarksopen,bookmarksnumbered,citecolor=red,urlcolor=red]{hyperref}
\usepackage{amsmath}
\usepackage{amsthm}
\usepackage[hide]{ed}
\usepackage[T1]{fontenc}
\usepackage{lmodern}
\usepackage{authblk}
\usepackage{mathrsfs,amsopn}
\usepackage{amssymb}
\usepackage{stmaryrd}
\usepackage{comment}
\usepackage{tikz}
  \usetikzlibrary{positioning,backgrounds,cd}
  \tikzset{line/.style={draw=black,line width=1pt}}
\usepackage[frozencache=true,cachedir=_minted]{minted}
\makeatletter
\renewcommand{\minted@error}[1]{\PackageWarning{minted}{#1}}
\makeatother
\usepackage[arrow,matrix]{xy}
\usepackage{quiver}
\usepackage{enumitem}
\usepackage{cmll}      
\usepackage{bbold}
\usepackage{booktabs}
\usepackage{longtable}
\usepackage{xspace}
\usepackage{pifont} 

\newlist{checklist}{itemize}{2}
\setlist[checklist]{label=$\square$}

\makeatletter
\renewcommand{\ednote}[2][\ednote@label]{\textcolor{red}{[#1: #2]}}
\makeatother

\newcommand\BibTeX{{\rmfamily B\kern-.05em \textsc{i\kern-.025em b}\kern-.08em
T\kern-.1667em\lower.7ex\hbox{E}\kern-.125emX}}

\newcommand{\minth}[1]{\mintinline{haskell}{#1}}
\newcommand{\mintp}[1]{\mintinline{python}{#1}}

\newcommand{\aggrA}{\mathrm{aggr}^\forall}
\newcommand{\aggrE}{\mathrm{aggr}^\exists}

\newcommand\sem[1]{\llbracket #1 \rrbracket}
\providecommand{\Set}{\mathbf{Set}}
\providecommand{\Meas}{\mathbf{Meas}}

\providecommand{\Stoch}{\mathbf{Stoch}}
\providecommand{\Kl}[1]{\mathcal{K}\!\ell\!\bigl(#1\bigr)}

\providecommand{\id}{\mathrm{id}}

\newcommand{\norm}[1]{\left\|#1\right\|}
\newcommand{\F}{\mathrm{F}}
\newcommand{\MF}{\mathrm{MF}}
\renewcommand{\P}{\mathrm{P}}
\newcommand{\MP}{\mathrm{MP}}

\newcommand{\abs}[1]{\vert #1 \vert}
\newcommand{\category}[1]{\ensuremath{\mathbf{#1}}\xspace}
\newcommand{\categoryobjects}[1]{\ensuremath{\abs{\category{#1}}}\xspace}

\newcommand{\DO}{{\textbf{do }}}
\newcommand{\doTerm}[3]{\DO\, #1\leftarrow #2; \ #3}
\newcommand{\doTwoTerm}[5]{\DO\, #1\leftarrow #2;#3\leftarrow #4; \ #5}

\DeclareMathOperator*{\argmax}{arg\,max}

\DeclareMathOperator{\expec}{\mathbb{E}}  

\newtheorem{definition}{Definition}
\newtheorem{example}{Example}

\newtheorem{proposition}{Proposition}
\newtheorem{remark}{Remark}

\begin{document}

\title{NeSyCat: A Monad-Based Categorical Semantics of the Neurosymbolic ULLER Framework}

\author[1]{Daniel Romero Schellhorn}
\author[2]{Till Mossakowski}
\affil[1]{University of Osnabr\"uck, Osnabr\"uck, Germany. \texttt{daniel.schellhorn@uni-osnabrueck.de}}
\affil[2]{University of Osnabr\"uck, Osnabr\"uck, Germany. \texttt{till.mossakowski@uni-osnabrueck.de}}

\date{}

\maketitle

\begin{abstract}
  \textbf{ULLER} (\underline{U}nified \underline{L}anguage for \underline{LE}arning and \underline{R}easoning)  offers a unified first-order logic (FOL) syntax, enabling its knowledge bases to be used directly across a wide range of neurosymbolic systems.
  The original specification endows this syntax with three pairwise independent semantics: classical, fuzzy, and probabilistic, each accompanied by dedicated semantic rules.  
  We show that these seemingly disparate semantics are all instances of one categorical framework based on \emph{monads}, the very construct that models side effects in functional programming.  
  This enables the \emph{modular} addition of new semantics and systematic translations between them. As example, we outline the addition of generalised quantification in Logic Tensor Networks (LTN) to arbitrary (also infinite) domains by extending the Giry monad to probability spaces. In particular, our approach allows a modular implementation of ULLER in Python and Haskell, of which we have published initial versions on GitHub. 
\end{abstract}


\section{Introduction}
\label{sec:intro}

Neurosymbolic integration is a rapidly developing branch of AI.
In the past, numerous heterogeneous approaches have emerged, each with its own code base. \cite{vankriekenULLERUnifiedLanguage2024} introduces ULLER, a unified neurosymbolic library that aspires to play for neurosymbolic systems the role that TensorFlow and PyTorch play for deep-learning workflows. Their theoretical core is the concept of a NeSy system: standard first-order logic enriched with neural components. In particular, formulas of the form
$$x:=m(T_1,\dots,T_n)F\qquad(T_i\text{ terms, }F\text{ a formula involving variable $x$},m\text{ a neural model})$$
are used to integrate neural models $m$ into logical formulas.
These formulas go beyond classical first-order logic.
Instead, they perform computations that may return multiple values and typically involve non-determinism or probability distributions as illustrated in the following toy example in \citep{vankriekenULLERUnifiedLanguage2024}:
\begin{example}[MNIST digit classification] \label{ex:mnist-uller}
$$\begin{array}{l}
  \forall x\in \mathtt{ImageData} \\
  \quad(n_1 := \mathtt{classify}(x.\mathtt{im}_{1}) \\
  \qquad(n_2 := \mathtt{classify}(x.\mathtt{im}_{2}) \\
  \quad\qquad(n_1+n_2=x.\mathtt{sum})~))
\end{array}$$
\end{example}
\noindent
This classifies two images of digits and checks whether the sum of the resulting numbers is as specified in the dataset. The resulting (e.g.\ fuzzy or probabilistic) truth value can be used in a loss function. A shorthand notation for this is:
$$\forall x\in \mathtt{ImageData} 
  \bigl(n_1 := \mathtt{classify}(x.\mathtt{im}_{1}), n_2 := \mathtt{classify}(x.\mathtt{im}_{2}) 
  ~(n_1+n_2=x.\mathtt{sum})~\bigr)$$

\noindent
To simplify notation and to stress the relation to dynamic logic \citep{harel2001dynamic,DBLP:journals/fac/MossakowskiSG10}, we henceforth use the notation
$$\forall x\in \mathtt{ImageData} \bigl[n_1 := \mathtt{classify}(x.\mathtt{im}_{1})\bigr]\bigl[n_2 := \mathtt{classify}(x.\mathtt{im}_{2})\bigr]n_1+n_2=x.\mathtt{sum}$$
or shorthand 
$$\forall x\in \mathtt{ImageData} \bigl[n_1 := \mathtt{classify}(x.\mathtt{im}_{1}),n_2 := \mathtt{classify}(x.\mathtt{im}_{2})\bigr]n_1+n_2=x.\mathtt{sum}$$

\noindent
While the notion of NeSy system in \cite{vankriekenULLERUnifiedLanguage2024} is very powerful, it also has several shortcomings:

\begin{itemize}
\item
  There is no uniform inductive definition of truth, i.e.\ of the truth value of a sentence in an interpretation. Rather, the notion of NeSy system has the inductive interpretation function as a component, meaning that classical, probabilistic and fuzzy NeSy systems employ three different inductive definitions of truth. Parts of these definitions of truth are copied verbatim from one NeSy system to another, other parts need to be replaced. This duplication of semantic rules is not modular. By contrast, we aim at a truly uniform inductive definition of truth value that is independent of the NeSy system and hence can be reused for different NeSy systems, such that the NeSy system itself is a parameter of the inductive definition of truth.
\item
  The case of continuous probability distributions (involving probability kernels or Markov kernels) is not covered faithfully, because in this case, measurable spaces are required to properly define the mentioned Markov kernels. However, measurable spaces are not considered in \cite{vankriekenULLERUnifiedLanguage2024} and probability measures are confused with density functions in the monadic formula of the probabilistic semantics.
\item
  The treatment of logical connectives is not uniform across NeSy systems, i.e.\ the different sets of connectives are not considered as instances of a common abstract (algebraic) notion. 
  Also, quantifiers in probabilistic semantics are defined using possibly infinite products, without requiring a suitable order structure on domains and without discussing convergence. 
\item
  The high‑level concepts of semantics and computation are not properly separated. Computation (namely sampling) is mixed into the semantics at least in two places. The first one in the classical semantics, where the possibly multi-valued $\argmax$ can only be properly evaluated using sampling. We conceptualise the $\argmax$ differently as a transition \emph{between} semantics. The second time is in "Sampling Semantics", which is not really semantics but computation (sampling).
\end{itemize}

\noindent
We argue that ULLER is conceptually robust and show that a monadic formulation resolves all of the foregoing issues.
In particular, ULLER formulas of the form
$$[x:=m(T_1,\dots,T_n)]F\qquad(T_i\text{ terms, }F\text{ a formula involving variable $x$},m\text{ a neural model})$$
can be modelled using Moggi's notion of computational monad \citep{moggiNotionsComputationMonads1991}, which has been introduced to model side effects in (functional) programming. Although monads originate in category theory, we first present them using a set-theoretic approach that does not involve any category theory. The generalisation to an arbitrary category, which is needed for continuous probabilities and some aspects of infinite domains, comes only in later sections. We call our categorical approach simply \textit{NeSyCat}, standing for \textit{Neurosymbolic Category-theory} framework. 

This paper is organised as follows. Section~\ref{sec:set-nesy-frameworks} introduces NeSy frameworks and their algebraic prerequisites without use of category theory. Based on that, section~\ref{sec:syntax-and-semantics}
introduces the syntax and semantics of NeSyCat. Section~\ref{sec:set-examples} discusses several examples of set-based NeSy frameworks, including classical, fuzzy, and probabilistic ones. Section~\ref{sec:nesy-trans} discusses translations between NeSy systems. Section~\ref{sec:cat-nesy} generalises  set-based NeSy frameworks to categorical NeSy frameworks, which is needed for continuous probability distributions and infinite domains. Section~\ref{sec:cat-sem} defines the categorical semantics of NeSyCat and provides examples of categorical NeSy systems.  Section~\ref{sec:related} discusses related work, and section~\ref{sec:implementation} outlines our implementation of NeSyCat in Python and Haskell. Section~\ref{sec:conclusion} concludes the paper and outlines future work.
Appendix~\ref{sec:cat-gen} contains a brief introduction to the needed concepts of category theory.  This paper is an extended an revised version of \cite{pmlr-v284-schellhorn25a}.

Finally, while this paper focuses on the theoretical foundation, questions of differentiability, computability, and computational complexity are left for future work in a subsequent paper dedicated to the implementation and practical evaluation of NeSyCat.

\section{Set-Based NeSy Frameworks}
\label{sec:set-nesy-frameworks}

\paragraph{Background: key concepts from ULLER.}
We briefly recall the main concepts from \cite{vankriekenULLERUnifiedLanguage2024} that our framework builds upon. ULLER is a first-order logic extended with \emph{statements} $x := f(T_1,\dots,T_n)(F)$, which bind a variable $x$ to the result of applying a (possibly neural) function $f$ to terms $T_i$, then evaluate the formula $F$ in the extended context. An \emph{interpretation} $I$ assigns meanings to the non-logical symbols: domains $I(D)$ to domain symbols, functions $I(f)$ to function symbols (as conditional probability distributions), predicates $I(P)$ to predicate symbols, and constants $I(c)$ to constant symbols. A \emph{NeSy system} $(I, \eta, \mathcal{B}, \sem{\cdot})$ in ULLER bundles an interpretation $I$, a variable assignment $\eta$, a set of outputs $\mathcal{B}$ (e.g., $\{0,1\}$ for classical, $[0,1]$ for fuzzy/probabilistic), and a semantic function $\sem{\cdot}$ that inductively assigns a value in $\mathcal{B}$ to each formula. ULLER provides three example semantics -- classical, probabilistic, and fuzzy -- each defining $\sem{\cdot}$ separately, which leads to the duplication of semantic rules that we address below.

A neurosymbolic framework (NeSy framework) is a general framework for NeSy systems combining neural models with symbolic logic, and it provides the semantic background for the specific logic involved.
Examples are the logics behind DeepProbLog \citep{DBLP:journals/ai/ManhaeveDKDR21} or Logic Tensor Networks \citep{DBLP:journals/ai/BadreddineGSS22}. 

The notion of NeSy framework is not defined in \cite{vankriekenULLERUnifiedLanguage2024}. Rather, they define a notion of NeSy system, which is quite ad-hoc, because it simultaneously makes two choices: (1) a choice of a particular interpretation with functions, predicates and neural models (e.g., probabilities for traffic lights, or neural networks learning addition of digit images), and (2) a choice of semantic rules for interpreting terms (which also involves a choice of the logic, e.g.\ classical or probabilistic or fuzzy). This causes semantic rule duplication.

Our notion of NeSy framework provides a means to disentangle these two choices. 
Moreover, our approach makes semantic rules independent not only of particular interpretations, but also of the choice of logic (classical, probabilistic, or fuzzy).

In the sequel, we first introduce the algebraic structure needed to model the space of truth values (section~\ref{subsec:2Mon-BLat}), generalising the standard two-valued truth value space $\{T,F\}$ used in classical logic.
Neural, probabilistic and/or non-deterministic aspects are incorporated into logics using monads, which are the key concept of our approach and which we introduce in section~\ref{subsec:monads}. Being such prepared, we can define a NeSy framework as consisting of a monad and a truth algebra in setion~\ref{subsec:nesy-frameworks}.

\subsection{Double Monoid Bounded Lattices (2Mon-BLat)}
\label{subsec:2Mon-BLat}
We need an algebraic structure to model the space of truth values. We weaken the notion of BL algebra of \cite{hajekMetamathematicsFuzzyLogic1998} from fuzzy logic as follows:

\begin{definition}\label{def:2Mon-BLat}
A \textbf{double monoid bounded lattice (2Mon-BLat)} $\mathcal{R}$ is a tuple
\[
  \bigl(S, \; \leq, \; \bot, \; \top, \; \otimes, \;0, \; 1, \;\oplus,  \;\to, \; \neg \bigr)
\]
in which $S$ is a set, $\mathcal{L}:=(S,\leq)$ a bounded lattice (see App.~\ref{sec:glossary} for this and other standard algebraic and categorical terms), while $\bot \in S$ and $\top \in S$ are its bottom and top elements.\footnote{In many cases we have $\bot$ is neutral element for $\oplus$ and $\top$ is neutral element for $\otimes$, for example inside of the unit interval $[0,1]$. In some cases, like G\"odel or Boolean logic, we even have $\oplus = \lor$ and $\otimes = \land$. Also, $\to$ is normally chosen as right adjoint to $\otimes$ (App.~\ref{sec:glossary}) or as $x \to y :=\neg x \oplus y$. In the first case $\neg$ can be defined as implication to zero, in the second one it is defined a priori. Check Table~\ref{tab:all-algebras} for details.} Also $(S,\otimes, 1)$ and $(S,\oplus, 0)$ are monoids (App.~\ref{sec:glossary}), $\to$ is a map $S \times S \to S$, and $\neg$ is a map $S \to S$.
\end{definition}

\noindent
\label{rem:uller-comparison}
Our monoid requirement on $(\otimes, 1)$ and $(\oplus, 0)$ -- i.e.\ associativity and the existence of neutral elements -- is stronger than what the original ULLER paper \citep{vankriekenULLERUnifiedLanguage2024} explicitly requires. ULLER instead requires that its semantics be \emph{classical in the limit}: when the truth space is restricted to $\{0,1\}$ (which actually are the implicit unit elements), the semantics should agree with classical first-order logic (see \citep[Appendix~C]{vankriekenULLERUnifiedLanguage2024}). This classical-in-the-limit property follows from the monoid axioms together with the standard boundary conditions of t-norms and t-conorms (e.g., $1 \otimes x = x$ and $0 \oplus x = x$), but is strictly weaker: one could have classical-in-the-limit behaviour without full associativity. 

We adopt the stronger monoid requirement because it provides a cleaner algebraic structure and associativity ensures that the order of evaluation of iterated conjunctions and disjunctions does not matter. All concrete algebras used in this paper (Table~\ref{tab:all-algebras}) satisfy these requirements; whether there are useful NeSy-relevant algebras that are classical in the limit but not monoids is an open question. Note that we do \emph{not} require commutativity of $\otimes$ or $\oplus$ in general, although all concrete instances in Table~\ref{tab:all-algebras} happen to be commutative.

Also, we want to allow different aggregation operations other than infinite meet and join to cover the quantifiers of Logic Tensor Networks \citep{DBLP:journals/ai/BadreddineGSS22}, motivating the following definition:\footnote{This is inspired by the notion of \emph{aggregated functions} in \citep{DBLP:conf/nesy/BadreddineS21}.}
\begin{definition}\label{def:aggregated-2Mon-BLat}
An \textbf{aggregated 2Mon-BLat (aggr-2Mon-BLat)} has for each set $X$ two order-preserving maps:
\begin{align*}
 \mathrm{aggr}_X^\forall, \mathrm{aggr}_X^\exists  &:\; \mathcal{L}^X \longrightarrow \mathcal{L}.
\end{align*}
In case of a complete lattice, $\mathrm{aggr}_X^\forall$ can be chosen as meet $\bigwedge_X$ and $\mathrm{aggr}_X^\exists$ as join $\bigvee_X$.

\end{definition}

\subsection{Set-Based Monads}
\label{subsec:monads}
We interpret formulas $[x:=m(T_1,\dots,T_n)]F$ involving neural models $m$ ($=$ certain computations) using Moggi's notion of computational monad \citep{moggiNotionsComputationMonads1991}. 
Moggi's central idea is to use monads to abstract from various computational effects by providing an abstract framework covering many different types of effects such as state, input, output or non-determinism. We here use monads in a similar way, but rather focus on neural and probabilistic computations. This use of moands has been widely adopted for probabilistic programming \citep{ramsey2002stochastic,staton2016semantics}; we discuss this connection in detail in Section~\ref{subsec:monadic-pp}. Note that the monadic approach does not focus on mere effectful computation only, it also captures values that are returned by such computations. E.g.\ think of a stateful function that alters the state but also returns a value, or, in our case, a probability distribution over a set of values.
A monad provides three ingredients: (1) given a set (type) of values, a set of computations over such values, (2) an embedding of values into computations, and (3) a notion of sequential composition of computations that allows for passing values through the composition.

\begin{definition} \label{def:set_Kleisli_triple}
A \textbf{(set-based) monad}\footnote{In category theory, a monad on a category $\mathcal{C}$ can be equivalently presented in two ways: (1)~as an endofunctor ${\mathcal T}$ with natural transformations $\eta: \mathrm{Id} \Rightarrow {\mathcal T}$ (unit) and $\mu: {\mathcal T}^2 \Rightarrow {\mathcal T}$ (multiplication) satisfying associativity and unit laws; (2)~as an \emph{extension system} (also called \emph{Kleisli triple}), which is the presentation we adopt here because it directly captures sequential composition of computations and is therefore most natural for computer science applications \citep{moggiNotionsComputationMonads1991}.} $({\mathcal T}, \eta, (-)^*)$  consists of:
\begin{itemize}
  \item A mapping ${\mathcal T}$, mapping sets $X$ to sets ${\mathcal T} X$ (of computations with values from $X$),
  \item A family of functions: $\eta_X : X \to {\mathcal T} X$ for each set $X$ (construing a value $a\in X$ as stateless computation $\eta_X(a)\in {\mathcal T} X$),
  \item A function that assigns to each function $f : X \to {\mathcal T} Y$ a function $f^* : {\mathcal T} X \to {\mathcal T} Y$ (called the \emph{Kleisli extension}), needed for sequential composition of computations, 
\end{itemize}
such that the following axioms hold:
\begin{enumerate}
  \item $(\eta_X)^* = \mathrm{id}_{{\mathcal T} X}$, the identity function on ${\mathcal T} X$,
  \item $f^* \circ \eta_X = f$ for all $f : X \to {\mathcal T} Y$,
  \item $(g^* \circ f)^* = g^* \circ f^*$ for all $f : X \to {\mathcal T} Y$ and $g : Y \to {\mathcal T} Z$.
\end{enumerate}
\end{definition}
\noindent
We now explain how the third ingredient of a monad can be used for sequential composition.
Assume that we have a computation $ma:{\mathcal T} A$ and a computation $f(x):{\mathcal T} B$ parameterised over $x:A$, i.e.\  $f:A\to \mathcal T B$. We can compose them to $(\lambda x:A.f(x))^*(ma)$\footnote{In this lambda notation the variable $x$ of type $A$ is sent to $f(x)$ of type ${\mathcal T} B$.}, or short $f^*(ma)$, of type ${\mathcal T} B$. In Haskell's do-notation,  $(\lambda x:A.f(x))^*(ma)$ is written as $\doTerm{x}{ma}{f(x)}$. For the sake of readability, we will use do-notation henceforth. $f^*(ma)$ is also written as $ma >\!>= f$ in fish notation.\footnote{See \url{https://ncatlab.org/nlab/show/monad+\%28in+computer+science\%29\#DoNotation} for a comprehensive comparison of all different notations.}
\begin{example}\label{ex:nonempty-powerset-monad}
  \leavevmode
  \textbf{Non-empty Powerset monad $\mathcal{P}_{\neq \emptyset}$}
  For a set $X$:
  \[
    \mathcal{P}_{\neq \emptyset} X := \{A \subseteq X \mid A \neq \emptyset\} \quad \text{(non-empty subsets of $X$)},
  \]
  \[
    \eta_X(x) := \{x\},~~~
    f^*(A) := \bigcup_{a \in A} f(a)
    \quad\bigl(f:X\to \mathcal{P}_{\neq \emptyset}Y, A \in \mathcal{P}_{\neq \emptyset}X\bigr).
  \]
  \end{example}

\noindent
This monad models non-deterministic computation returning one or more possible values. The Kleisli extension $f^*(A) = \bigcup_{a\in A} f(a)$ applies $f$ to each element and collects all results -- a ``flatMap'' over sets. In a NeSy setting, this captures situations such as an ambiguous image classifier that returns multiple plausible labels (e.g., a poorly written digit classified as both 7 and 1).

\begin{example}\label{ex:dist-monads}
\leavevmode
\textbf{Probability distribution monad \(\mathcal D\).}\footnote{Note that the sums below are only finite if one excludes all the zero addenda.}
\[
  \mathcal D X := \Bigl\{\,
     \rho:X\to[0,1] \text{ finitely supported} \;\Bigm|\;
     \sum_{x\in X}\rho(x)=1
   \Bigr\} \quad \text{(prob. distributions on $X$)},
\]
\[
  \eta_X(x):=\delta_x, \; \delta_x(y)=\left\{\begin{array}{l}1,\ x=y\\0,\ x\not=y\end{array}\right.\ 
  f^{*}(\rho)(y):=\sum_{x\in X}\!f(x)(y)\cdot\rho(x)
  \quad\bigl(f:X\!\to\!\mathcal D Y,\; \rho\in\mathcal D X\bigr).
\]
$\delta_x$ is the probability distribution that assigns all probability mass to $x$.
$f^{*}(\rho)$ corresponds to a two-level random process: first $x$ is drawn from $\rho$, then $y$ is drawn from $f(x)$. This results in a marginal distribution of $Y$ for the joint distribution $\hat\rho(x,y):= f(x)(y) \cdot \rho(x)$.
$\doTerm{x}{ma}{f(x)}$ can be interpreted as ``sample $x$ from $ma$ and then proceed with $f(x)$''.
When the formula $F:=[x:=m(T)]P(x)$ evaluates to a numerical truth value, $f^*(\rho)$, with $f(a):=\sem{P(a)}$ and $\rho:=m(T)$, computes the expected truth value under $\rho$ -- this is how probabilistic NeSy semantics yields expected-value loss functions (cf.\ Example~\ref{ex:mnist-uller}).
\end{example}

\subsection{Definition of Set-Based NeSy Framework}
\label{subsec:nesy-frameworks}

Given some basic set of truth values $\Omega$, our NeSy systems work on the monadic space of truth values ${\mathcal T}\Omega $, which is required to be an aggregated 2Mon-BLat. If ${\mathcal T}$ is the identity monad, ${\mathcal T}\{ 0,1\}$ is just the two-element set $\{ 0,1\}$ of classical truth values. 
If ${\mathcal T}$ is the distribution monad, ${\mathcal T}\{ 0,1\}$ is isomorphic to the unit interval $[0,1]$, regarded as the space of probabilistic or fuzzy truth values.

\begin{definition}\label{def:set-nesy-framework}
A \textbf{NeSy framework}  $\mathcal{F} = ({\mathcal T},\mathcal R)$ consists of
\begin{enumerate}
  \item a \emph{monad} ${\mathcal T}$,
  \item an \emph{aggr-2Mon-BLat} $\mathcal R$ on ${\mathcal T}\Omega$ for some set $\Omega$.
\end{enumerate}
Here, $\Omega$ is a set acting as truth basis,\footnote{Similar to the basis of a vector space.} and ${\mathcal T}\Omega$ is the monadic space of truth values.
\end{definition}

\noindent
Examples are given in Table~\ref{tab:set-nesy-frameworks} and discussed in more detail in section~\ref{sec:set-examples}. Note that further examples arise by varying the 2Mon-BLat $\mathcal R$ on $[0,1]$. Examples requiring category theory are given in Table~\ref{tab:basic-nesy-frameworks}. 
\begin{table}[h]
\centering
\caption{NeSy Framework Examples (set-based)}
\label{tab:set-nesy-frameworks}
\begin{tabular}{@{}llllll@{}}
\toprule
\textbf{Logic/Theory} & \textbf{${\mathcal T}$} & \textbf{$\Omega$} & \textbf{${\mathcal T}\Omega$} & \textbf{$\mathcal R$} & \textbf{Subsection} \\ \midrule
Classical                       & Identity                                        & $\{0,1\}$          & $\{0,1\}$            & Boolean Alg. & \ref{subsec:3val-sem} \\
Three-valued LP                    & Powerset $ \mathcal{P}_{\neq \emptyset}$                                    & $\{0,1\}$          & $\{0,B,1\}$            & Kleene/Priest Alg. & \ref{subsec:3val-sem} \\
Distributional            & Distribution $\mathcal D$                       & $\{0,1\}$          & $[0,1]$              &Product BL–Alg. & \ref{subsec:dist-sem} \\
Finitary $\text{LTN}_p$          & Distribution $\mathcal D$                        & $\{0,1\}$          & $[0,1]$              &Product SBL-Alg. & \ref{par:LTN} \\
Classical Fuzzy              & Identity                                        & $[0,1]$            & $[0,1]$              &Classical BL–Alg. & -- \\
\bottomrule
\end{tabular}
\end{table}

\begin{proposition}\label{prop:lifting}
  Assume that we can lift lattices along ${\mathcal T}$, i.e.\ for a lattice structure on $X$, we can construct a lattice on ${\mathcal T} X$ such that $\eta$ preservers $\bot$ and $\top$.

  
If ${\mathcal R}=(\Omega,\leq,\bot,\top,\otimes,0,1,\oplus,\to,\neg)$ is a \emph{2Mon-BLat}, then ${\mathcal T}\Omega$ is so, too, in a canonical way.
\end{proposition}
\begin{proof}
We define $({\mathcal T}\Omega,\leq',\bot',\top',\otimes',0,1,\oplus',\to',\neg')$ as follows: 
\begin{itemize}
  \item $\leq'$ is the lifting of $\leq$,
  \item $\bot' := \eta(\bot), \qquad \top' := \eta(\top)$,
  \item $0' := \eta(0)$,
  \item $1' := \eta(1)$,
  \item $a \otimes' b := \doTwoTerm{x}{a}{y}{b}{\eta(x\otimes y)}$,
  \item $a \oplus' b := \doTwoTerm{x}{a}{y}{b}{\eta(x\oplus y)}$,
  \item $a \to' b := \doTwoTerm{x}{a}{y}{b}{\eta(x\to y)}$,
  \item $\neg'a := \doTerm{x}{a}{\eta(\neg x)}$.
\end{itemize}
Associativity and unit laws of $\otimes'$ and $\oplus'$ follows from the corresponding properties of the monoids for $\otimes$ and $\oplus$ and those of the monad.
\qed

\end{proof}

\section{Syntax and Semantics of NeSyCat}
\label{sec:syntax-and-semantics}

\subsection{Syntax of First-Order Logic}\label{subsec:fol-syntax}
The ULLER language of \cite{vankriekenULLERUnifiedLanguage2024} features computational function symbols that can be realised e.g.\ by neural networks. In a similar spirit, we here add computational predicate symbols, which are also realised by neural networks, for example in Logic Tensor Networks  \citep{DBLP:journals/ai/BadreddineGSS22}.

\begin{definition}
  A \textbf{NeSy signature} $\Sigma$ consists of 
  \begin{itemize}
  \item a set $S$ of \emph{sorts} of $\Sigma$,
  \item two disjoint sets $\mathrm{Pred}$, $\mathrm{mPred}$ of \emph{predicate symbols} and \emph{computational predicate symbols} of form $p: s_1, \ldots, s_n$, where $p$ is a name and each $s_i\in S$ a sort, 
  \item two disjoint sets $\mathrm{Func}$, $\mathrm{mFunc}$ of \emph{function symbols} and \emph{computational function symbols} of form $f: s_1, \ldots, s_n \to s$, , where $f$ is a name and $s,s_i\in S$ are sorts.
  \end{itemize}
 (Computational) predicate symbols with no arguments are called (computational) propositional symbols $(\mathrm{Prps}$ \text{and} $\mathrm{mPrps} \text{ respectively})$. 
 Function symbols with one argument are called properties $(\mathrm{Prop})$, those with none are called constants $(\mathrm{Const})$.
  \end{definition}

\subsection{Tarskian Semantics}\label{subsec:tarsk-sem}

\begin{definition}
\label{def:nesy-interpretation}
A \textbf{NeSy interpretation} $\mathcal I$ on $\Sigma$ of $({\mathcal T}, \mathcal R)$, for a \emph{NeSy signature} $\Sigma$ and a \emph{NeSy framework} $({\mathcal T}, \mathcal R)$, is given by
\begin{itemize}
 \item  a set $\mathcal I(s)$ for every sort $s$, 
 \item a function $\mathcal I(f) : \mathcal I(s_1) \times \ldots \times \mathcal I(s_n) \to \mathcal I(s)$ for every (normal) function symbol $f : s_1, \ldots, s_n \to s\in\mathrm{Func}$,
 \item a function $\mathcal I(m) : \mathcal I(s_1) \times \ldots \times \mathcal I(s_n) \to {\mathcal T}(\mathcal I(s))$ for every computational function symbol $m : s_1, \ldots, s_n \to s\in\mathrm{mFunc}$,
 \item a function $\mathcal I(P) : \mathcal I(s_1) \times \ldots \times \mathcal I(s_n) \to \Omega$ for every predicate symbol $P : s_1, \ldots, s_n\in\mathrm{Pred}$,
 \item and a function $\mathcal I(M) : \mathcal I(s_1) \times \ldots \times \mathcal I(s_n) \to {\mathcal T}\Omega$ for every computational predicate symbol $M : s_1, \ldots, s_n\in\mathrm{mPred}$.
 \end{itemize}
\end{definition}

\begin{definition}\label{def:nesy-system}
A \textbf{NeSy system} can be defined as a triple $(\mathcal T, \mathcal R, \mathcal I)$, where $(\mathcal T, \mathcal R)$ is a \emph{NeSy framework} and $\mathcal I$ a \emph{NeSy interpretation} over that same \emph{NeSy framework}. 

Note that in this notation the underlying set of basic truth values $\Omega$, with which the aggr-2Mon-BLat $\mathcal R$ is defined, and the signature $\Sigma$ on which the interpretation $\mathcal I$ is defined, are suppressed since they are implicit.
\end{definition}

\begin{example}\label{ex:comp-vs-noncomp}
The MNIST addition formula from the introduction (Example~\ref{ex:mnist-uller}, slightly adapted),
$$\forall x\in \mathtt{ImageData} \bigl[n_1 := \mathtt{classify}(x_1),n_2 := \mathtt{classify}(x_2)\bigr]n_1+n_2=\mathtt{sum}(x_1,x_2)$$
illustrates why the signature requires both computational and non-computational symbols.
Here, $\mathtt{classify}:\mathtt{Image}\to\mathtt{Digit}$ is a \emph{computational} function symbol ($\in\mathrm{mFunc}$): it is realised by a neural network and its interpretation maps into $\mathcal D(\mathcal I(\mathtt{Digit}))$, i.e.\ it returns a distribution over digits rather than a single digit.
By contrast, $+:\mathtt{Digit}^2\to\mathtt{Digit}$ ($\in\mathrm{Func}$), and the projection symbols $\pi_1, \pi_2:\mathtt{ImageData}\to\mathtt{Image}: x \mapsto x_1, x \mapsto x_2$ and $\mathtt{sum}:\mathtt{ImageData}\to\mathtt{Digit}$ are likewise normal function symbols. Similarly, $= \ : \mathtt{Digit}^2\to\mathtt{Bool}$ is the built-in equality predicate symbol interpreted in the usual way as diagonal.

The distinction between computational and non-computational symbols matters because the monadic machinery of a NeSy framework acts only on computational symbols: the monad $\mathcal T$ wraps their output, and the dynamic logic brackets $[\cdot]$ bind the resulting values. Non-computational symbols, by contrast, are interpreted as ordinary functions or predicates and are independent of the choice of monad.
\end{example}

\noindent
Compared to \cite{vankriekenULLERUnifiedLanguage2024}, we have added computational predicate symbols, because LTN and other NeSy frameworks use these. However note that for probabilistic logic and weighted model counting, ULLER makes (probabilistic) independence assumptions due to the nature of its notion of interpretation.\footnote{As already hinted at in the original ULLER paper \citep{vankriekenULLERUnifiedLanguage2024} and looked at in more detail by \cite{kriekenIndependenceAssumptionNeurosymbolic2024} and \cite{kriekenNeurosymbolicReasoningShortcuts2025}, the independence assumption can prevent NeSy predictors from correctly modelling uncertainty.} While computational function symbols enable the use of conditional probabilities, computational predicate symbols are not available in ULLER.\footnote{Note that the predicate $\mathit{True}(x)$, used in some ULLER examples to convert values $\{0,1\}$ (delivered by computational funnctions) into truth values, can also be used in NeSyCat. However, in NeSyCat, we can directly use computational predicates instead.} Hence, ULLER supports a certain combination of probabilistic and fuzzy logic, and so do LTNs. For details, see section~\ref{subsec:cont-prob-sem}

Also, we have dropped uniformity of the notion of interpretation---it now becomes dependent on the monad at hand. This is necessary for faithfully distinguishing finitely supported and continuous probability distributions and for dealing with LTN-style quantification on infinite domains. Still, computational symbols can be realised by neural networks in all of these cases. However, the details of the mapping from neural networks to interpretations of computational symbols differ.

In \cite{vankriekenULLERUnifiedLanguage2024}, based on an interpretation, the notion of NeSy system provides a Tarskian inductive definition $\sem{\cdot}$ of the semantics of formulas and thus it implicitly also defines the semantics of the logical symbols. The drawback of this approach is that the Tarskian semantics $\sem{\cdot}$ is inherently tied to the specific NeSy system. 

We can modularise matters here, because we first give a semantics of the logical symbols via a NeSy framework, and based on that, the interpretation provides the semantics of the non-logical symbols. Hence, the inductive definition of the Tarskian semantics $\sem{\cdot}$ needs to be given only once, and this definition holds across all NeSy frameworks and systems. 

The Tarskian semantics uses variable valuations. As a preparation, we collect variables occurring freely in terms and formulas. Given a term $T$, let $\Gamma_T=\{x_1:s_1,\ldots,x_n:s_n\}$ be its context (of variables) and $s_T$ be its sort. Then, we define the space of variable valuations compatible with $T$ as $\mathcal{V}_T:=\prod_{x:s \in \Gamma_T} \mathcal{I}(s)=\mathcal{I}(s_1)\times\cdots\times\mathcal{I}(s_n)$. 
Given a variable valuation $\nu\in \mathcal{V}_T$, note that is maps variables $x:s$ to values in $\mathcal{I}(s)$.

Analogously, for a formula $F$, $\mathcal{V}_F:=\prod_{x:s \in \Gamma_F} \mathcal{I}(s)$. Sentences, i.e.\ formulas without free variables, can be interpreted over the unique empty variable valuation (in this case, $\mathcal{V}_F=\{()\}$ is a singleton set containing just the empty tuple). Hence, as usual in first-order logic, for interpreting sentences, we do not need a variable valuation.

\begin{definition}\label{def:Tarskian-semantics} 
The \textbf{Tarskian semantics} $\sem{\cdot}$ of formulas and terms in a \emph{NeSy system} $(\mathcal T, \mathcal R, \mathcal I)$ is given by two functions defined in Table~\ref{tab:semantics}:
\begin{align*}
\textbf{Formulas: }\sem{F}_{\mathcal I,\nu}: \mathcal{V}_F \to {\mathcal T} \Omega, \quad \textbf{Terms: }\sem{T}_{\mathcal I,\nu}: \mathcal{V}_T \to \mathcal{I}(s_T).
\end{align*}

\begin{table}[h]
\centering
\caption{Inductive definition of the Tarskian semantics\label{tab:semantics}}
\vspace{0.4em}
\label{tab:semantics}
\renewcommand{\arraystretch}{1.2}
\begin{tabular}{@{}ll@{}}
\toprule
\textbf{Syntax} &
\textbf{Set Semantics $\sem{\cdot}_{\mathcal I,\nu}$} \\
\midrule
\multicolumn{2}{@{}l}{\textbf{Terms}}\\ \midrule

$\displaystyle 
\sem{x:s}$ & $\displaystyle \nu(x)$\\

$\displaystyle 
\sem{c}, \ \sem{T.\mathsf{prop}}, \ \sem{f(\vec{T})}$ & $\displaystyle\mathcal I(c), \ \mathcal I(\mathsf{prop})(\sem{T}) ,\ \mathcal I(f)\bigl(\sem{\vec{T}}\bigr)$\\

\midrule
\multicolumn{2}{@{}l}{\textbf{Atomic formulas}}\\ \midrule
$\displaystyle 
\sem{R}, \ \sem{P(\vec{T})}$ & $\displaystyle \eta_\Omega(\mathcal I(R)),\ \eta_\Omega\bigl(\mathcal I(P)(\sem{\vec{T}})\bigr)$\\
$\displaystyle 
\sem{N}, \ \sem{M(\vec{T})}$ & $\displaystyle\mathcal I(N), \ \mathcal I(M)\bigl(\sem{\vec{T}}\bigr)$\\

\midrule
\multicolumn{2}{@{}l}{\textbf{Compound formulas}}\\ \midrule
$\displaystyle 
\sem{\bot}, \ \sem{\top}$ & 
$\displaystyle 
\bot_{\mathcal R}, \ \top_{\mathcal R}$ \\

$\displaystyle 
\sem{F\to G}, \ \sem{\neg F}$ &
$\displaystyle
\sem{F}\to_{\mathcal R}\sem{G}, \ \neg_{\mathcal R}\sem{F}$ \\[4pt]

$\displaystyle 
\sem{F\| G},\ \sem{F\& G}$ &
$\displaystyle
\sem{F}\oplus_{\mathcal R}\sem{G},\ \sem{F}\otimes_{\mathcal R}\sem{G}$ \\[4pt]

$\displaystyle 
\sem{\exists x{:}s\,F}, \ \sem{\forall x{:}s\,F}$ &
$\displaystyle
\mathrm{aggr}^\exists_{\mathcal I(s)}(\lambda a. \sem{F}_{\nu[x\mapsto a]}), \ \mathrm{aggr}^\forall_{\mathcal I(s)}(\lambda a. \sem{F}_{\nu[x\mapsto a]})$ \\[4pt]

$\displaystyle 
\sem{[x := m(\vec{T})]F}$ &
$\displaystyle
\doTerm{a}{\mathcal{I}(m)(\sem{\vec{T}})}{\sem{F}_{\nu[x\mapsto a]}}$ \\[4pt]
\bottomrule
\end{tabular}

$\vec{T}$ stands for $T_1,\ldots,T_n$.
\end{table}
\noindent
We write $\sem{T}_{\mathcal I,\nu}=\sem{T}_{\mathcal I}(\nu)$ and $\sem{F}_{\mathcal I,\nu}=\sem{F}_{\mathcal I}(\nu)$. That said, we mostly omit $\mathcal I$ and $\nu$ if clear from the context.
\end{definition}

\paragraph{The computational formula rule.}
The last row of Table~\ref{tab:semantics} is the key rule connecting neural models to logic. It reads: ``to evaluate $[x := m(\vec{T})]F$, first evaluate the terms $\vec{T}$, then apply the computational function $m$ to obtain a computation $\mathcal{I}(m)(\sem{\vec{T}}) \in {\mathcal T}(\mathcal{I}(s))$. The do-notation $\doTerm{a}{\cdot}{\cdot}$ then `extracts' a value $a$ from this computation and evaluates $F$ with $x$ bound to $a$.'' What ``extract'' means concretely depends on the monad: for the distribution monad $\mathcal{D}$ (Ex.~\ref{ex:dist-monads}), it means averaging $\sem{F}_{\nu[x\mapsto a]}$ over all possible values $a$, weighted by their probabilities; for the non-empty powerset monad (Ex.~\ref{ex:nonempty-powerset-monad}), it means collecting the truth values of $F$ for every possible value $a$ returned by $m$. This is precisely the Kleisli extension $\sem{F}_{\nu[x\mapsto a]}^*$ of the monad applied to $\mathcal{I}(m)(\sem{\vec{T}})$.

\paragraph{Quantifier aggregation.}
The quantifiers $\forall x{:}s$ and $\exists x{:}s$ evaluate $F$ for every element $a$ of the domain $\mathcal{I}(s)$, producing a family of truth values $\bigl(\sem{F}_{\nu[x\mapsto a]}\bigr)_{a \in \mathcal{I}(s)}$. The aggregation functions $\aggrA_{\mathcal{I}(s)}$ and $\aggrE_{\mathcal{I}(s)}$ then combine this family into a single truth value. Unlike in classical first-order logic, where the quantifiers are hardcoded as infimum and supremum, the aggregation in NeSyCat is parametrised by the NeSy framework: for classical logic, it is $\min$/$\max$; for probabilistic semantics, a product and probabilistic sum; for LTN, a $p$-norm. Section~\ref{sec:set-examples} spells out each of these instantiations.

\noindent
\label{rem:inductive-semantics}
The Tarskian semantics of Definition~\ref{def:Tarskian-semantics} is deliberately inductive, i.e.\ defined by recursion on the structure of formulas. In ULLER \citep{vankriekenULLERUnifiedLanguage2024}, the broader notion of ``neurosymbolic system'' also permits non-inductive evaluation strategies---for instance, first compiling a formula into a circuit representation and then evaluating that circuit. NeSyCat does not incorporate such strategies into its semantics. This is by design: we cleanly separate the \emph{meaning} of a formula (given inductively by the Tarskian semantics) from the \emph{method of evaluation} (which may involve compilation, circuit construction, or sampling). Non-inductive evaluation strategies are thus implementation concerns that operate \emph{on top of} the inductively defined semantics, rather than alternatives to it. This mirrors the classical separation in programming language theory between denotational semantics and compilation, and aligns with the analogous separation of semantics and sampling discussed in Section~\ref{subsec:sampling}.

\section{Examples of Set-Based Semantics}\label{sec:set-examples}

In the sequel, we will discuss some NeSy frameworks in more detail and spell out how the semantic rules look when instantiated. Each subsection specifies a monad ${\mathcal T}$ and a 2Mon-BLat algebra $\mathcal{R}$; the full definitions of all algebras used in this paper are collected in Table~\ref{tab:all-algebras} (Sec.~\ref{sec:algebraic-overview}). Note that the semantic rule for computational formulas $[x:=m(\vec{T})]F$ is always determined by the monad's Kleisli extension $f^*$, while the rules for connectives and quantifiers are determined by the 2Mon-BLat. We often implicitly define parts of the 2Mon-BLat through the semantic rules. E.g.\ the $\aggrE$ and $\aggrA$ functions are implicitly defined by listing semantic rules for the quantifiers.

\subsection{Classical and Three-valued Semantics}\label{subsec:3val-sem}

Classical semantics is simply given by the identity monad (Ex.~\ref{ex:identity-monad}) and the Boolean algebra (Table~\ref{tab:all-algebras}, row ``Boolean'') on $\Omega=\{0,1\}$, which results in classical first-order logic. 

Our classical semantics is deterministic, while \cite{vankriekenULLERUnifiedLanguage2024} use probability distributions, causing the need for selection of values with highest probability, done via argmax. We will model this as NeSy transformation in section~\ref{sec:nesy-trans} and need a non-deterministic NeSy framework as target of this transformation. 
The (non-empty) powerset monad models non-deterministic computations, cf.\ multialgebras \citep{walicki1994multialgebras}. These result in non-deterministic truth values as in:

\paragraph{Logic of Paradox Semantics}\label{par:lp-sem}
For the Logic of Paradox \citep{priest2008introduction}, we use the non-empty powerset monad $\mathcal{P}_{\neq \emptyset}$ (Ex.~\ref{ex:nonempty-powerset-monad}) with the Kleene/Priest algebra (Table~\ref{tab:all-algebras}, row ``Priest''). We have ${\mathcal T} = \mathcal{P}_{\neq \emptyset}$, $\Omega = \{0,1\}$ (equivalently $\{F,T\}$), and $\mathcal T \Omega = \mathcal{P}_{\neq \emptyset}(\{0,1\}) = \{\{0\}, \{1\}, \{0,1\}\}$. The three truth values correspond to: $\{0\} \equiv F$ (false only), $\{1\} \equiv T$ (true only), and $\{0,1\} \equiv B$ (both true and false). Following the uniform Tarskian semantics (where the union in the computational formula rule arises from the Kleisli extension of $\mathcal{P}_{\neq\emptyset}$, cf.\ Ex.~\ref{ex:nonempty-powerset-monad}):

\begin{align}
  \sem{[x:=m(\vec{T})]F}
      &:=
     \bigcup_{a \in \mathcal{I}(m)(\sem{\vec{T}})} \sem{F}_{\nu[x\mapsto a]} \\
      \sem{\exists x{:}s\;F} 
      &:= \sup_{a \in \mathcal{I}(s)}\sem{F}_{\nu[x\mapsto a]}, \quad
      \sem{\forall x{:}s\;F}
      := \inf_{a \in \mathcal{I}(s)}\sem{F}_{\nu[x\mapsto a]} \label{eq:lp-quants}
      \\
    \sem{F \| G}
      &:= \max(\sem{F}, \sem{G}),
  \quad
    \sem{F \& G}
      := \min(\sem{F}, \sem{G})
      \\[4pt]
    \sem{F \rightarrow G}
      &:= \max( \sem{\neg F}, \sem{G}), \quad
    \sem{\neg F}
      := \begin{cases}
        \{0,1\} & \text{if } \sem{F} = \{0,1\} \\
        (\{0,1\} \setminus \sem{F} ) & \text{else }
        \end{cases}
      \\[4pt]
    \sem{\bot}
      &:= \{0\}, \quad
    \sem{\top}
      := \{1\}
  \end{align}

\noindent
where the operations implement Priest's Logic of Paradox with the lattice ordering $\{0\} <_{\text{LP}} \{0,1\} <_{\text{LP}} \{1\}$ (i.e., F $<$ B $<$ T). 

\subsection{Distributional Semantics}\label{subsec:dist-sem}

The distributional semantics corresponds to the third row in Table~\ref{tab:set-nesy-frameworks}, where we use the distribution monad $\mathcal{D}$ (Ex.~\ref{ex:dist-monads}) over the classical truth basis $\{0,1\}$, yielding the truth space $[0,1]$ equipped with the Product algebra (Table~\ref{tab:all-algebras}, row ``Product''): $x\otimes y = xy$ (product t-norm), $x\oplus y = x+y-xy$ (probabilistic sum), $x\to y = \max(1, y/x)$ (R-implication), and $\neg x = 1-x$. This framework provides the semantic foundation for probabilistic logic programming systems and neural-symbolic approaches that work with probability distributions over truth values. The weighted sum in the computational formula rule below arises from the Kleisli extension of $\mathcal{D}$ (cf.\ Ex.~\ref{ex:dist-monads}).
In this setting, computational predicates and function symbols return probability distributions rather than deterministic values. We need to restrict interpretations to finite domains. Finite quantification just iterates conjunction or disjunction. For infinite quantifiers, check section~\ref{subsec:cont-prob-sem}.

\begin{align}
  \sem{[x:=m(\vec{T})]F}
      &:=
      \sum_{a \in \mathcal{I}(s_m)} \sem{F}_{\nu[x\mapsto a]} \cdot \rho_{m}( a \mid \vec{T}),\mbox{~~~ where }\rho_{m}( a \mid \vec{T}) := \mathcal{I}(m)(\sem{\vec{T}})(a)  \\
      \sem{\exists x{:}s\;F} 
    &:= 1 - \prod_{a \in \mathcal{I}(s)} (1 - \sem{F}_{\nu[x\mapsto a]}), \quad
      \sem{\forall x{:}s\;F}
      := \prod_{a \in \mathcal{I}(s)} \sem{F}_{\nu[x\mapsto a]} \\
    \sem{F \| G}
      &:= \sem{F}  + \sem{G} -  \sem{F}\cdot \sem{G},
  \quad
    \sem{F \& G}
      := \sem{F} \cdot \sem{G}
      \\[4pt]
    \sem{F \rightarrow G}
      &:= 
            \max\bigl(1, \sem{G}\slash\sem{F} \bigr),
      \quad
    \sem{\neg F}
      := 1 - \sem{F}
      \\[2pt]
    \sem{\bot}
      &:= 0, \quad
    \sem{\top}
      := 1.
  \end{align}

\subsection{Finitary $\text{LTN}_p$ Semantics}
\label{par:LTN}

The finitary Logic Tensor Networks (LTN) semantics corresponds to the fourth row in Table~\ref{tab:set-nesy-frameworks}, employing the distribution monad $\mathcal{D}$ (Ex.~\ref{ex:dist-monads}) over the classical truth basis $\{0,1\}$ with the truth space $[0,1]$ equipped with the S-Product algebra (Table~\ref{tab:all-algebras}, row ``S-Prod.''). This algebra has the same t-norm ($x\otimes y = xy$) and t-conorm ($x\oplus y = x+y-xy$) as the Product algebra used in the distributional semantics (Sec.~\ref{subsec:dist-sem}), but uses S-implication $x \to y = 1-x+xy$ (strong implication) instead of R-implication $\max(1, y/x)$ (residual implication). Moreover, aggregation differs as well: following \cite{DBLP:journals/ai/BadreddineGSS22}, quantification is performed using $p$-norms, where the parameter $p$ controls the "softness" of the logical operations. For existential quantification, we compute the $p$-norm of truth values, while for universal quantification, we use the dual formulation $1 - \|1-\cdot\|_p$:

\begin{align}
      \sem{\exists x{:}s\;F}
      &:= \Bigl(\frac{1}{| \mathcal{I}_s |}\sum_{a \in \mathcal{I}_s} \sem{F}_{\nu[x\mapsto a]}^{\,p}\Bigr)^{\!1/p},
      \\
      \sem{\forall x{:}s\;F}
      &:= 1 - \Bigl(\frac{1}{| \mathcal{I}_s |}\sum_{a \in \mathcal{I}_s} (1-\sem{F}_{\nu[x\mapsto a]})^{\,p}\Bigr)^{\!1/p},
      \\
    \sem{F \rightarrow G}
      &:= 
            1 - \sem{F} + \sem{F}\cdot\sem{G}.
  \end{align}

\noindent
This is however only possible for finite domains. For the infinite case and additional quantifier variants, we refer to section~\ref{subsec:inf-LTN}.

\subsection{Sampling "Semantics"}\label{subsec:sampling}
While \cite{vankriekenULLERUnifiedLanguage2024} have introduced a sampling semantics, we think that the semantics should define probabilities, while an implementation can work with e.g.\ Monte Carlo sampling in order to obtain an approximation that is easier to implement (and, in the case of quantification over infinite domains, unavoidable). Hence, we do not discuss sampling semantics here.
But we expect that a Monte Carlo convergence theorem can be stated and proved.

\section{NeSy Transformations}
\label{sec:nesy-trans}

\begin{figure}[h]
\centering
\begin{tikzpicture}[
  node distance=2cm and 1.5cm,
  every node/.style={
    rectangle,
    draw=black,
    fill=white,
    thick,
    minimum width=2.5cm,
    minimum height=0.8cm,
    align=center,
    font=\footnotesize
  },
  arrow/.style={
    ->,
    thick,
    >=stealth
  }
]

\node (A) {Crossing $x$};

\node (B) [below=0.5cm of A] {$\mathrm{argmax}[\mathrm{light}(x)] = \{{\color{red}\bullet}, {\color{orange}\bullet}, {\color{green}\bullet}\}$};

\node (R) [below left=1cm and 1cm of B] {$l = \color{red}\bullet$};
\node (AM) [below=1cm of B] {$l = \color{orange}\bullet$};
\node (G) [below right=1cm and 1cm of B] {$l = \color{green}\bullet$};

\node (Rd) [below=0.75cm of R] {$\mathrm{argmax}[\mathrm{drive}(x,{\color{red}\bullet})] = {\color{red}\times}$};
\node (Rg) [below=0.75cm of Rd] {$d = {\color{red}\times} $\\Formula TRUE};

\node (ASet) [below=0.75cm of AM] {$\mathrm{argmax}[\mathrm{drive}(x,{\color{orange}\bullet})] = \{{\color{red}\times},{\color{green}\checkmark}\}$};
\node (A0) [below=0.75cm of ASet, xshift=-1.4cm] {$d = {\color{red}\times} $\\Formula FALSE};
\node (A1) [below=0.75cm of ASet, xshift=1.4cm] {$d = {\color{green}\checkmark} $\\Formula TRUE};

\node (Gd) [below=0.75cm of G] {$\mathrm{argmax}[\mathrm{drive}(x,{\color{green}\bullet})] = {\color{green}\checkmark}$};
\node (Gg) [below=0.75cm of Gd] {$d = {\color{green}\checkmark}$ \\Formula TRUE};

\node (Comb) [below=1cm of A1, xshift=-1.4cm] {Union of Branches = BOTH};

\draw[arrow] (A) -- (B);
\draw[arrow] (B) -- (R);
\draw[arrow] (B) -- (AM);
\draw[arrow] (B) -- (G);
\draw[arrow] (R) -- (Rd);
\draw[arrow] (Rd) -- (Rg);
\draw[arrow] (AM) -- (ASet);
\draw[arrow] (ASet) -- (A0);
\draw[arrow] (ASet) -- (A1);
\draw[arrow] (G) -- (Gd);
\draw[arrow] (Gd) -- (Gg);
\draw[arrow] (Rg) -- (Comb);
\draw[arrow] (A0) -- (Comb);
\draw[arrow] (A1) -- (Comb);
\draw[arrow] (Gg) -- (Comb);

\end{tikzpicture}
\caption{Argmax transformation flowchart for the traffic light example from \citep{vankriekenULLERUnifiedLanguage2024}, shown for a fixed crossing $x$. The diagram instantiates only the local non-deterministic evaluation of $[l:=\text{light}(x), \ d:=\text{drive}(x,l)](d\neq{\color{green}\checkmark}, l={\color{green}\bullet})$ after the $\argmax$ transformation; it does not yet apply the universal quantifier. Since the traffic lights are uniformly distributed, $\argmax[\text{light}(x)]$ returns all three colours, and in the amber case $\argmax[\text{drive}(x,l)]$ returns both ${\color{red}\times}$ and ${\color{green}\checkmark}$. The LP truth values generated along these branches are collected by the local monadic evaluation, yielding BOTH. If one evaluates the full formula $\forall x{:}\text{Crossing} \ [l:=\text{light}(x), \ d:=\text{drive}(x,l)](d\neq{\color{green}\checkmark}, l={\color{green}\bullet})$, the quantifier only afterwards aggregates these per-crossing results over the domain of crossings.}
\label{fig:argmax-diagram}
\end{figure}
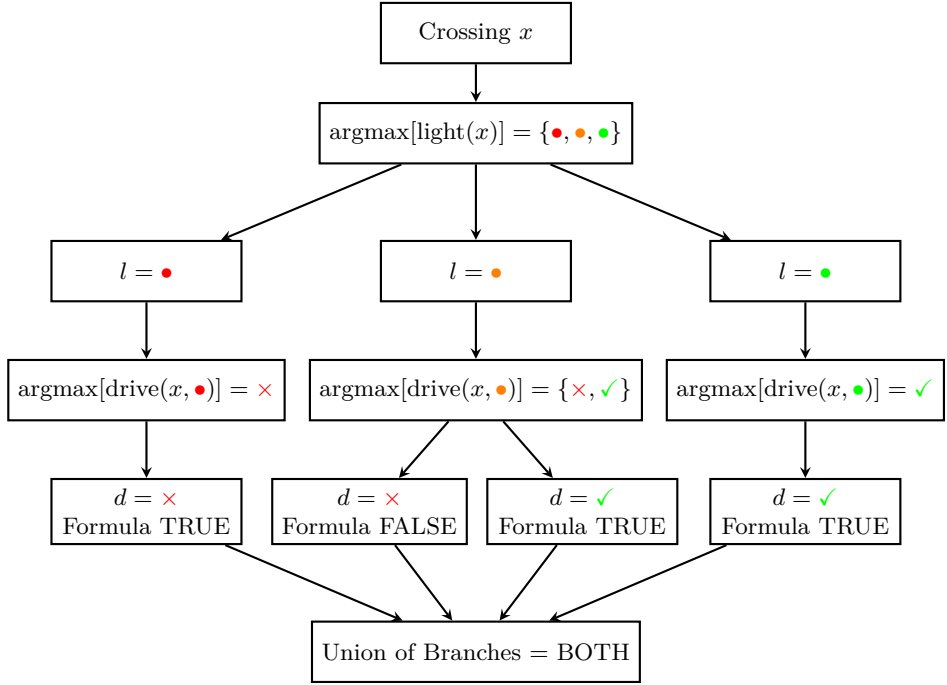

\noindent
The original ULLER paper \citep{vankriekenULLERUnifiedLanguage2024} uses a uniform notion of interpretation. This leads to the problem that classical semantics needs to extract values with maximal probability (using arg max) from a distribution, which is not possible if there is a tie\footnote{Or in case of an infinite distribution (see section~\ref{subsec:cont-prob-sem})}. Here, we propose an alternative way of dealing with this problem: namely, using a NeSy transformation, we can move e.g.\ from an interpretation in a probabilistic NeSy framework to one in a classical framework. Dealing with ties can be done using a non-deterministic semantics.

\begin{definition}\label{def:nesy-transformation}
 A \textbf{NeSy transformation} $\alpha: \mathcal{F} \to \mathcal{F}'$ between two NeSy frameworks is a family of functions\footnote{For those interested in category theory; this is in fact a natural transformation (App.~\ref{sec:glossary}), hence the name.} $\alpha_\Sigma : \mathrm{Intp}_\mathcal{F}(\Sigma) \to \mathrm{Intp}_\mathcal{F'}(\Sigma)$. Here, the
 \emph{interpretation function} $\mathrm{Intp}_\mathcal{F}:
 \Sigma \;\mapsto\; \mathrm{Intp}(\mathcal{F},\Sigma)$ sends a signature to the set of interpretations on $\mathcal{F}$ for that signature. We write $\alpha$ for $\alpha_\Sigma$ if $\Sigma$ is clear from context.

\end{definition}

\paragraph{Argmax transformation: From distributional to non-deterministic semantics}
\label{par:argmax}
For a given distributional interpretation $\mathcal{I}(m)$ of a computational function symbol $m$, we can define a non-deterministic interpretation $\alpha(\mathcal{I})(m)$ of $m$ by defining, where $s_m$ is the sort of $m$, and likewise, $M$ is a computational predicate symbol:
$$ \alpha(\mathcal{I})(m):= \argmax_{a \in \mathcal{I}(s_m)}\bigl[ \mathcal{I}(m)(a) \bigr], \qquad \alpha(\mathcal{I})(M):= \argmax_{b \in \Omega = \{0, 1\}} \bigl[ \mathcal{I}(M)(b) \bigr], $$ 

\noindent
and for all other symbols set $\alpha(\mathcal{I}):= \mathcal{I}$. This definition is \emph{not} possible in the general probabilistic case, because probability measures often return zero on \emph{all} single values. It is also a non-deterministic interpretation since it returns a \emph{set} of values instead of a single value. The resulting semantics is three-valued, as in section~\ref{subsec:3val-sem}, and we apply it to the traffic light example from \citep{vankriekenULLERUnifiedLanguage2024} in Fig.~\ref{fig:argmax-diagram}. There we fix a single crossing $x$ and evaluate the computational part $[l:=\text{light}(x), \ d:=\text{drive}(x,l)](d\neq{\color{green}\checkmark}, l={\color{green}\bullet})$: the uniform distribution on $\text{light}(x)$ makes the first $\argmax$ return all three colours, and the amber branch remains non-deterministic because $\argmax[\text{drive}(x,{\color{orange}\bullet})]$ returns both ${\color{red}\times}$ and ${\color{green}\checkmark}$. The local monadic evaluation, via the Kleisli extension for the non-empty powerset monad, then collects both truth values, giving the LP value BOTH. The universal quantifier in the full formula is not used inside the diagram; it would only aggregate these already computed per-crossing values over the domain of crossings. This $\argmax$ transformation is just one example of many possible NeSy transformations.

Then, in the practical implementation of NeSyCat, one can use random sampling (see section~\ref{subsec:sampling}) over a uniform distribution to obtain a single value from the set of values $\alpha(\mathcal{I})(m)$. This gives a precise foundation for the use of $\argmax$ in the classical semantics in \cite{vankriekenULLERUnifiedLanguage2024}.

\section{Categorical NeSy Frameworks}
\label{sec:cat-nesy}

So far, in this paper, we have not made use of category theory. Indeed, we have introduced set-based notions of monad and of Double monoid bounded lattice.  They do not rely on category theory, nor do the central definitions of NeSyCat, in particular, the notions of NeSy framework, of signature, interpretation and formula, nor the rules of the Tarskian semantics. 

By contrast, in this and the next section, we will heavily make use of category theory. (The subsequent Sections~\ref{sec:related} and~\ref{sec:implementation} implicitly use the results of this section, but do not directly make uses of category theory. Reader not interested in category theory may skip to Section~\ref{sec:related}.) The main purpose of our use of category theory is the possibility to work with continuous probability distributions.\footnote{The lack of which is described as a major shortcoming of traditional NeSy systems by \cite{smetNeuralProbabilisticLogic2023}.} Another motivation is the need for quantification structures for infinite domains, which are common in first-order logic. In these cases, we need to work with structured objects and structure-preserving maps, like measurable spaces and measurable functions. Without organising such spaces and maps into a category, we still can use most of the set-theoretic rules of the semantics. However, we would need to restrict to structure-preserving maps in places like the definition of interpretation (Def.~\ref{def:nesy-interpretation}), and, more severely,
would need to prove that the semantic rules in Def.~\ref{def:Tarskian-semantics} again yield structure-preserving maps. When using category theory, we can avoid such proofs and base our theory on a certain structure that is required for the involved categories. 

That said, we still can use the set-theoretic semantics rules for convenience, also for categories of sets with structure\footnote{Technically, these set-based categories $\category{C}$ are constructs in the sense of \cite{AHS}, which means they come with a faithful functor $U:\category{C}\to\category{Set}$. $U$ maps objects of $\category{C}$ to their underlying sets.}, just because the categorical version of the rules tells us that the resulting maps will be structure-preserving. The only set-theoretic rules of the semantics that we cannot simply re-use are the rules for quantifiers. Here, the aggregation functions $\aggrE$ and $\aggrA$ generally need to make use of the extra structure that the ``sets with structure'' have, see e.g.\ the aggregation functions for probabilistic semantics defined in section~\ref{subsec:cont-prob-sem} below need to make use of the probability distribution coming with the universe.



A brief introduction to category theory is given in Appendix~\ref{sec:cat-gen}. 

\subsection{More (on) Monads}
\label{subsec:more-monads}

\begin{definition}
  A \textbf{monad} on a category $\mathcal{C}$ (in extension system form, cf.\ the footnote in Def.~\ref{def:set_Kleisli_triple}) involves a mapping of objects ${\mathcal T}: \mathrm{Ob}(\mathcal{C}) \to \mathrm{Ob}(\mathcal{C})$, a family of morphisms $\eta_X : X \to {\mathcal T} X$ for each object $X$ in $\mathcal{C}$ (called the \emph{unit}), and a function that assigns to each morphism $f : X \to {\mathcal T} Y$ a morphism $f^* : {\mathcal T} X \to {\mathcal T} Y$
  such that the axioms hold as in Def.~\ref{def:set_Kleisli_triple}. Note that a set-based monad is then a monad on the category $\Set$. 
\end{definition}
  
\begin{definition}[Strong monad]
  A monad $({\mathcal T},\eta,(\!-\!)^{\!*})$ on a  
  category $\mathcal{C}$ with finite products is called \textbf{strong} if there is a natural
  transformation
  \[
  \mathcal{S}_{A,B} : A \times {\mathcal T} B \;\longrightarrow\; {\mathcal T}(A \times B)
  \]
  satisfying the naturality condition: for any morphisms $f: A \to A'$ and $g: B \to B'$,
  \[
  {\mathcal T}(f \times g) \circ \mathcal{S}_{A,B} = \mathcal{S}_{A',B'} \circ (f \times {\mathcal T} g),
  \]
  and such that
  \[
  \mathcal{S}_{A,B}\circ(\mathrm{id}_A\times\eta_B)=\eta_{A\times B},
  \qquad
  \mathcal{S}_{A,B}\circ(\mathrm{id}_A\times f^*) = (id \times f)^*\circ \mathcal{S}_{A,C}.
  \]
  \end{definition}

\begin{example}\label{ex:identity-monad}
\leavevmode
\textbf{Identity monad $\mathrm{Id}$ on $\mathbf{Set}$ (category of sets and functions).}
For a set $X$:
\[
  \mathrm{Id} (X) := X \quad \text{(identity functor)},
\]
\[
  \eta_X(x) := x,~~~
  f^*(x) := f(x)
  \quad\bigl(f:X\to Y, x \in X\bigr).
\]
\end{example}

\noindent
The identity monad models deterministic computation with no effects -- the baseline case. The Kleisli extension $f^*(x) = f(x)$ is just ordinary function application; composition in the Kleisli category reduces to ordinary function composition. In a NeSy setting, this corresponds to classical two-valued logic where neural models return crisp values (e.g., a thresholded classifier emitting exactly one label).

\begin{example}\label{ex:powerset-monad}
\leavevmode
\textbf{Powerset monad $\mathcal{P}$ on $\mathbf{Set}$ (category of sets and functions).}
For a set $X$:
\[
  \mathcal{P} X := \{A \subseteq X\} \quad \text{(powerset of $X$)},
\]
\[
  \eta_X(x) := \{x\},~~~
  f^*(A) := \bigcup_{a \in A} f(a)
  \quad\bigl(f:X\to \mathcal{P}Y, A \subseteq X\bigr).
\]
\end{example}

\noindent
Unlike the non-empty variant (Ex.~\ref{ex:nonempty-powerset-monad}), the full powerset monad allows the empty set as a valid outcome, representing failure or unsatisfiability. The Kleisli extension works identically -- union over mapped elements -- but if $f$ returns $\emptyset$ for some input, that branch contributes nothing. In a NeSy setting, this models a reasoning system where queries may have no satisfying assignment.

\begin{example}[Sub-Distribution Monad $\mathcal S$] The sub-distribution monad $\mathcal S$ is similar to the distribution monad $\mathcal D$ but it allows for finitely supported measures that do not sum up to 1, that means:\label{def:sub-monad}
  \[
    \mathcal S X := \Bigl\{\,
        \rho:X\to[0,1]\;\Bigm|\;
        \text{$\sum_{x\in X}\rho(x)\leq1$, $\rho$ countably additive and has finite support}
      \Bigr\},
  \]
  \[
    \eta_X(x):=\delta_x, \delta_x(y)=\left\{\begin{array}{l}1,\ x=y\\0,\ x\not=y\end{array}\right.\ 
    f^{*}(\rho)(y):=\sum_{x\in X}\!\rho(x)f(x)(y)
    \quad\bigl(f:X\!\to\!\mathcal S Y,\;\rho\in\mathcal S X\bigr).
  \]
  \end{example}

\noindent
The sub-distribution monad models probabilistic computation where probability mass can be ``lost'' (total $\leq 1$). The Kleisli extension works like that of $\mathcal{D}$, but the resulting probabilities may sum to less than~1, representing partial failure or absorption. In a NeSy setting, this captures a probabilistic classifier with a rejection option -- if no class exceeds a confidence threshold, some probability mass is discarded rather than assigned.

\begin{example}\label{ex:giry-monad}
\leavevmode
\textbf{Giry monad \(\mathcal G\) on \(\Meas\), the category of measurable spaces and measurable maps (see Appendix~\ref{sec:cat-gen}).}
For a measurable space \((X,\Sigma_X)\) and Dirac measure $\delta_x$ on $X$ for $x \in X$:
\[
  \mathcal G(X,\Sigma_X)
    := \bigl\{\rho:X \to [0,1]  \Bigm| \rho(X) = 1,\ \rho \text{ countably additive} \bigr\} \quad \text{(prob. measures on $X$)},
\]
\[
  \eta_{(X,\Sigma_X)}(x):=\delta_x,~~~  \delta_x(A)=\left\{\begin{array}{l}1,\ x\in A\\0,\ x\not\in A\end{array}\right.
\]
\[  
  f^{*}(\rho)(A):=\int\limits_{X}\!f(x)(A) d\rho(x)
  \quad\bigl(f:X\!\to\!\mathcal G Y , A\subseteq Y\text{ measurable}\bigr).
\]
$\delta_x$ is the probability distribution that assigns all probability mass to $x$.
Note that the Kleisli category (App.~\ref{sec:glossary}) $\Kl{\mathcal{G}}$ is isomorphic to the category $\Stoch$ of measurable spaces and Markov kernels, which is a Markov category in the sense of \cite{fritz_synthetic_2020}; see Section~\ref{subsec:monadic-pp} for further discussion.
\end{example}

\noindent
The Kleisli extension $f^*(\rho)(A) = \int_X f(x)(A)\, d\rho(x)$ computes the probability of the outcome landing in $A$, averaged over all inputs $x$ weighted by~$\rho$ -- the continuous analogue of $\mathcal{D}$'s ``two-level random process.'' In a NeSy setting, this models e.g.\ a sensor where the input (temperature, position) is a continuous random variable and the neural model maps it to a probability distribution over outcomes.

\begin{example}\label{ex:inf-giry-monad}
\leavevmode
\textbf{Infinite Giry monad $\mathcal G_{\infty}$ on \(\Meas\).}
For a measurable space $(X,\Sigma_X)$:
\[
  \mathcal G_{\infty}(X,\Sigma_X)
    := \bigl\{\mu:\Sigma_X \to [-\infty,\infty] \mid \mu(\emptyset)=0,\;\mu \text{ countably additive}\bigr\},
\]
\[
  \eta_{(X,\Sigma_X)}(x):=\delta_x,\qquad
  \delta_x(A)=\begin{cases}1 & x\in A \\ 0 & x\notin A\end{cases},
\]
\[
  f^{*}(\mu)(A):=\int_{X}\!f(x)(A)\, d\mu(x)
  \quad\bigl(f:X\!\to\!\mathcal S Y,\, A\subseteq Y \text{ measurable}\bigr).
\]
Here the integral is the Lebesgue--Stieltjes integral with respect to the extended signed measure $\mu$. Writing the Jordan decomposition $\mu=\mu^{+}-\mu^{-}$ and using linearity of the integral, one checks that the monad laws hold; thus $\mathcal S$ extends the Giry monad by allowing negative and (possibly) infinite total mass.
\end{example}
  
\noindent
This monad generalises $\mathcal{G}$ by allowing signed and infinite measures, which is primarily a theoretical extension. It is needed for certain mathematical constructions (e.g., characteristic functions, reward/penalty signals in reinforcement-learning-style objectives) and as an intermediate step in proofs involving measure-theoretic arguments.

\begin{proposition}[and definition of the measure-space monad $\mathcal{M}$]\label{prop:measure-monad}
  We can define a measure monad $\mathcal{M}$ as a monad $({\mathcal T},\eta,\mu)$\footnote{Here we use the multiplication form of a monad $({\mathcal T},\eta,\mu)$, where $\mu$ is monad multiplication; see the footnote in Def.~\ref{def:set_Kleisli_triple} for the equivalence of presentations.} on the category $\mathbf{Measr}$ of measure spaces (see Appendix~\ref{sec:cat-gen} for the distinction between $\Meas$ and $\mathbf{Measr}$) with $\eta, \mu$ being the unit and multiplication of the Giry monad $\mathcal{G}$ on $\mathbf{Meas}$. For $\rho$ being probability measures, we can define a probability-space monad $\mathcal{O}$ on the category $\mathbf{Prob}$ of probability spaces (Appendix~\ref{sec:cat-gen}). The same construction can be applied to obtain an infinite measure-space monad $\mathcal M_{\infty}$ on $\mathbf{Measr}$.
  \begin{align*}
    \mathcal{M}((X, \rho)) &:= (\mathcal G(X), \rho^\eta),\\
    \rho^\eta &:= B \longmapsto \rho(\eta^{-1}(B)), \text{ for } B\subseteq \mathcal G(X) \text{ measurable},\\
    \eta^\mathcal{M} &:= \eta, \quad \mu^\mathcal{M} := \mu.
  \end{align*}
  \end{proposition}
  
  \begin{proof}
    If $\rho$ is a probability measure, we know that $\rho^{\eta}(\mathcal{G}(X))=\rho(\eta^{-1}(\mathcal{G}(X)))=\rho(X)=1$, countable additivity follows alike.
  \end{proof}

\noindent
The measure-space monad $\mathcal{M}$ models probability on spaces that already carry their own measure -- this is needed when quantifiers use sort-dependent measures (cf.\ Table~\ref{tab:prob-cats}, column +DQ). In a NeSy setting, this arises when the domain of discourse (e.g., images, sensor readings) comes equipped with a data distribution, and universal or existential quantification integrates against that distribution.

\subsection{The 2Mon-BLat Algebra: A Comprehensive Overview}
\label{sec:algebraic-overview}

In this subsection, we provide a comprehensive overview of the different algebraic structures that can serve as algebras on the truth space ${\mathcal T}\Omega$ in our neurosymbolic framework, along with their associated operations types and logical properties. 

Table~\ref{tab:all-algebras} provides a comprehensive comparison of the fundamental operations across different algebraic structures, showing how each algebra defines its basic operations.

\begin{table}[h!]
\centering
\caption{Overview of aggregated 2Mon-BLat}
\label{tab:all-algebras}
\small
\begin{tabular}{|l|c|c|c|c|c|c|c|c|}
\hline
\textbf{Algebra} & \textbf{Set} & $\bot$ & $\top$ & $\oplus$ & $\otimes$ & $\to$ & $\neg$ & \raisebox{-2pt}{$\aggrE$} \\
\hline
Boolean& $\{0,1\}$ & $0$ & $1$ & $\max$ & $\min$ & $I_{B}$ & $\neg_R$ & $\sup$ \\
\hline
$\text{LTN}_p$ & $[0,1]$ & $0$ & $1$ & $S_P$ & $T_P$ & $I_{SP}$ & $\neg_C$ & $\norm{\cdot}_p$ \\
\hline
$\text{LTN}_q$ & $[0,1]$ & $0$ & $1$ & $S_P$ & $T_P$ & $I_{SP}$ & $\neg_C$ & $P\exists_q$ \\
\hline
Product& $[0,1]$ & $0$ & $1$ & $S_P$ & $T_P$ & $I_P$ & $\neg_R$ & \raisebox{-1pt}{$P\exists$} \\
\hline
S-Prod.& $[0,1]$ & $0$ & $1$ & $S_P$ & $T_P$ & $I_{SP}$ & $\neg_C$ & \raisebox{-1pt}{$P\exists$}\\
\hline
Priest& $\{F,B,T\}$ & $F$ & $T$ & $\max$ & $\min$ & $I_{KD}$ & $\neg_C$ & $\sup$ \\
\hline
\end{tabular}
\end{table}

\paragraph{t-conorms and t-norms ($\oplus$ and $\otimes$, App.~\ref{sec:glossary}):}
\begin{itemize}
\item $S_P$ (Probabilistic sum): $x S_P y = x + y - xy$
\item $T_P$ (Product): $x T_P y = xy$
\end{itemize}

\paragraph{Implications ($\to$):}
\begin{itemize}
\item $I_P$ (Product/Goguen): $I_P(x,y) = \begin{cases} 1 & \text{if } x \leq y \\ y/x & \text{otherwise} \end{cases}$
\item $I_{B}$ (Boolean): $I_{B}(x,y) = \begin{cases} 0 & \text{if } x=1, y=0 \\ 1 & \text{otherwise} \end{cases}$
\item $I_S$ (General S-implication): $I_S(x,y) = \neg x \oplus y$
\item $I_{KD}$ (Kleene-Dienes/Material): $I_{KD}(x,y) = \max(1-x, y)$
\item $I_{SP}$ (S-Product): $I_{SP}(x,y) = 1-x+xy$
\end{itemize}

\paragraph{Negations ($\neg$):}
\begin{itemize}
\item $\neg_R$ (Residual): $\neg_R x = x \to 0$ (includes Heyting/intuitionistic negation)
\item $\neg_C$ (Classical/1-Involutive): $\neg_C x = 1-x$
\item $\neg_V$ (0-Involutive): $\neg_V x = 0-x = -x$
\end{itemize}

\paragraph{Aggregations ($\aggrE$):}
\begin{itemize}
\item $P\exists$ (Infinitary Probabilistic Sum): $P\exists(x) = 1-\exp \Bigl( \expec_{a \sim \mu} \Bigl[\ln(1-\sem{F}_{\nu[x\mapsto a]}) \Bigr] \Bigr)$
\item $P\exists_q$ ($(\mu, q)$-approximated $P\exists$): $P\exists_q(x) = 1- \exp \Bigl( \expec_{a \sim \mu} \Bigl[ \ln\bigl(1-\sem{F}_{\nu[x\mapsto a]}\bigr)^q \Bigr]^{\frac{1}{q}} \Bigr) \\ \text{for } 1/2 \leq q \leq 1$, with $P\exists_q \to P\exists$ as $q \to 1$. Here
$\mu$ can be any measure and it depends on the context, in $\text{LTN}_q$ it depends on the measure space of the sort of the quantified variable at hand.
\end{itemize}
The quantification aggregations employ logarithmic and exponential transforms because they provide the natural generalisation of the product (or probabilistic sum) to infinitary domains. While finite probabilistic sums can be computed directly using products, extending to infinite domains requires the use of expectations, and the logarithmic and exponential transforms enable this generalisation while preserving the essential structure of probabilistic aggregation.

\subsection{Definition of Categorical NeSy Frameworks}

\begin{definition}
  \label{def:2Mon-BLat-cat}
\end{definition}

\begin{definition}

  An \textbf{aggr-2Mon-BLat internal to a category with finite products $\mathcal{C}$} 
on an object $A$ in $\mathcal{C}$ consists of a lattice on A internal\footnote{For an explanation check \url{https://ncatlab.org/nlab/show/internalization} and \url{https://ncatlab.org/nlab/show/monoid+in+a+monoidal+category}, since this is out of scope for this paper.} to $\mathcal{C}$, morphisms $\oplus,\otimes, \to:A\times A\to A$; $\bot,\top, 0,1:1_\mathcal{C}\to A$ and for any objects $B$ and $C$ maps $\mathrm{aggr}_{B,C}^\forall, \mathrm{aggr}_{B,C}^\exists:\mathcal{C}(B\times C,A)\to\mathcal{C}(B,A)$,  such that the axioms of Def.~\ref{def:2Mon-BLat} hold when appropriately interpreted in $\mathcal{C}$\footnote{$1_\mathcal{C}$ is a terminal object.}. 
\end{definition}
Note that our categorical handling of aggregation differs from that in the set-theoretic setting. A full analogy to the set-theoretic case would require aggregation morphisms $A^X\to A$, which would need a Cartesian closed category. This requirement seems too strong for our purposes, since it is not met in many examples and only is required for interpreting higher-order logics. However, given a Cartesian closed category (like $\category{Set}$) with aggregation $aggr:A^X\to A$, we can define aggregation  in the sense of Def.~\ref{def:2Mon-BLat-cat} as mapping $f:B\times C\to A$ to $B\stackrel{\Lambda(f)}{\to} A^C\stackrel{aggr}{\to} A$. Here, $\Lambda(f)$ is currying, defined as follows in \category{Set}: 
$\Lambda(f)(x)(y) = f(x,y)$. In the sequel, we will rely on this definition also for \category{Set}-based categories (constructs \citep{AHS}) that are not Cartesian closed, noting that the definition works even if $A^C$ is just a set and not an object in the category. Hence, in examples, we will define aggregation mostly as in Def.~\ref{def:2Mon-BLat}, but in some cases, we make use of the structure of the object $C$.

\begin{definition}\label{def:nesy-framework}
  A \textbf{NeSy framework}  $({\mathcal T} ,\mathcal R)$ consists of
  \begin{enumerate}
    \item a \textbf{strong monad} ${\mathcal T}$ with strength $\mathcal S$ on a category with finite products $\mathcal C$,\footnote{Note that $\Set$ is a Cartesian closed category (App.~\ref{sec:glossary}), and every monad on $\Set$ is strong.}
    \item an \textbf{aggr-2Mon-BLat} $\mathcal R$ in $\mathcal{C}$ on ${\mathcal T}\Omega$ for some object $\Omega$.
  \end{enumerate}
  Here, $\Omega$ is a set acting as truth basis,\footnote{Similar to the basis of a vector space.} and ${\mathcal T}\Omega$ is the monadic space of truth values.
  \end{definition}
  
  \noindent
  Examples are given in Table~\ref{tab:basic-nesy-frameworks} and their semantics are discussed in section~\ref{sec:cat-sem}. Note that further examples arise by varying the 2Mon-BLat $\mathcal R$ on $[0,1]$.

  \begin{table}[h]
  \centering
  \caption{NeSy Framework Examples (categorical)}
  \label{tab:basic-nesy-frameworks}
  \begin{tabular}{@{}lllllll@{}}
  \toprule
  \textbf{Logic/Theory} & \textbf{$\mathcal C$} & \textbf{${\mathcal T}$} & \textbf{$\Omega$} & \textbf{${\mathcal T}\Omega$} & \textbf{$\mathcal R$} & \textbf{Sem.} \\ \midrule
  Simple Prob.            & $\mathbf{Meas}$ &Giry $\mathcal G$                               & $\{0,1\}$          & $[0,1]$              &Product BL–Alg. & \S\ref{subsec:cont-prob-sem} \\
Standard Borel          & $\mathbf{BorelMeas}$ &Giry $\mathcal G|_{\mathbf{BorelMeas}}$                               & $\{0,1\}$          & $[0,1]$              &Product BL–Alg. & \S\ref{subsec:cont-prob-sem} \\
  Probabilistic             & $\mathbf{Measr}$ &Measr-space $\mathcal M$                               & $\{0,1\}$          & $[0,1]$              &Product BL–Alg. & \S\ref{subsec:cont-prob-sem} \\
  Infinitary $\text{LTN}_p$                     & $\mathbf{Prob}$ & Prob-space $\mathcal O$                                    & $\{0,1\}$          & $[0,1]$            &Product SBL-Alg. & \S\ref{subsec:inf-LTN}\\
  $\text{STL}_r$                      & $\mathbf{Measr}$ & $\infty$-Measr-space $\mathcal M_{\infty}$                                 & $\{1\}$          & $\overline{\mathbb{R}}$            & approx. $\mathbb{R}$ & \S\ref{subsec:stl-sem} \\
  \bottomrule
  \end{tabular}
  \end{table}

\section{Categorical Semantics}
\label{sec:cat-sem}

The categorical notion of interpretation differs from the set-theoretic definition (Definition~\ref{def:nesy-interpretation}) only in that \emph{sets} are replaced by \emph{objects} in the category $\mathcal{C}$ and \emph{functions} are replaced by \emph{morphisms} in $\mathcal{C}$. This generalisation allows the framework to work in any category with suitable structure, not just the category of sets and functions. In particular, $\mathcal{I}(s)$ is now an object in $\mathcal{C}$, and consequently $\id_{\mathcal{I}(s)}$ is the identity morphism on that object in $\mathcal{C}$ (not in $\Set$). In the set-theoretic case, $\mathcal{C} = \Set$, so the two notions coincide; in the categorical case, e.g.\ $\mathcal{C} = \Meas$, $\mathcal{I}(s)$ is a measurable space and $\id_{\mathcal{I}(s)}$ is the identity measurable function.

\begin{definition}[Tarskian semantics $\sem{\cdot}$ of formulas]\label{def:Tarskian-semantics-cat} 
  Given a NeSy framework $({\mathcal T}, \mathcal R)$ and a NeSy interpretation $\mathcal I$ we can determine the interpretation morphisms:
  \begin{align*}
  \textbf{Formulas: }\sem{F}_{\mathcal I}: \mathcal{V}_F \to {\mathcal T} \Omega, \quad \textbf{Terms: }\sem{T}_{\mathcal I}: \mathcal{V}_T \to \mathcal{I}(s_T):
  \end{align*}

  \begin{table}[h]
  \centering
  \caption{Inductive definition of the Tarskian semantics}
  \vspace{0.4em}
  \label{tab:semantics-cat}
  \renewcommand{\arraystretch}{1.2}
  \begin{tabular}{@{}lll@{}}
  \toprule
  \textbf{Syntax} &
  \textbf{Categorical Semantics} $\sem{\cdot}_{\mathcal I}$&
  \textbf{Set Semantics $\sem{\cdot}_{\mathcal I,\nu}$} \\
  \midrule
  \multicolumn{3}{@{}l}{\textbf{Terms}}\\ \midrule

  $\displaystyle 
  \sem{x:s}$ & $\id_{\mathcal I(s)}$ & $\displaystyle \nu_s(x)$\\ 

  $\displaystyle 
  \sem{c}$ & $\mathcal I(c)$ & $\displaystyle\mathcal I(c)$\\
  
  $\displaystyle 
  \sem{T.\mathsf{prop}}$ &
  $\displaystyle 
  \mathcal I(\mathsf{prop})\circ\sem{T}$ &
  $\displaystyle
  \mathcal I(\mathsf{prop})(\sem{T})$ \\
  
  $\displaystyle 
  \sem{f(\vec{T})}$ &
  $\displaystyle 
  \mathcal I(f)\circ\langle \sem{\vec{T}}_i\!\circ\!\pi_i\rangle_i$
  &
  $\displaystyle
  \mathcal I(f)\bigl(\sem{\vec{T}}\bigr)$ \\
  
  \midrule
  \multicolumn{3}{@{}l}{\textbf{Atomic formulas}}\\ \midrule
  $\displaystyle 
  \sem{P}$ & $\eta_\Omega\circ\mathcal I(P)$ & $\displaystyle\eta_\Omega(\mathcal I(P))$\\
  $\displaystyle 
  \sem{N}$ & $\mathcal I(N)$ & $\displaystyle\mathcal I(N)$\\
  $\displaystyle 
  \sem{R(\vec{T})}$ &
  $\displaystyle 
  \eta_\Omega\circ\mathcal I(R)\circ\langle \sem{\vec{T}}_i\!\circ\!\pi_i\rangle_i$
  &
  $\displaystyle
  \eta_\Omega\bigl(\mathcal I(R)(\sem{\vec{T}})\bigr)$ \\
  $\displaystyle 
  \sem{M(\vec{T})}$ &
  $\displaystyle 
  \mathcal I(M)\circ\langle \sem{\vec{T}}_i\!\circ\!\pi_i\rangle_i$
  &
  $\displaystyle
  \mathcal I(M)\bigl(\sem{\vec{T}}\bigr)$ \\
  
  \midrule
  \multicolumn{3}{@{}l}{\textbf{Compound formulas}}\\ \midrule
  $\displaystyle 
  \sem{\top}, \sem{\bot}$ & 
  $\displaystyle 
  1_{\mathcal R}, 0_{\mathcal R}$ & 
  $\displaystyle 
  1_{\mathcal R}, 0_{\mathcal R}$ \\[4pt]
  
  $\displaystyle 
  \sem{\neg F}$ &
  $\displaystyle 
  \neg_\mathcal R \circ \sem{F}$ &
  $\displaystyle
  \neg_\mathcal R(\sem{F})$ \\[4pt]
  
  $\displaystyle 
  \sem{F\to G}$ &
  $\displaystyle 
  \to_{\mathcal R}\!\circ
  \langle \sem{F}\!\circ\!\pi_F,\sem{G}\!\circ\!\pi_G\rangle$
  &
  $\displaystyle
  \sem{F}\,\to_{\mathcal R}\,\sem{G}$ \\[4pt]

  $\displaystyle 
  \sem{F\| G}$ &
  $\displaystyle 
  \oplus_{\mathcal R}\!\circ\!
  \langle \sem{F}\!\circ\!\pi_F,\sem{G}\!\circ\!\pi_G\rangle$
  &
  $\displaystyle
  \sem{F}\,\oplus_{\mathcal R}\,\sem{G}$ \\[4pt]
  
  $\displaystyle 
  \sem{F\& G}$ &
  $\displaystyle 
  \otimes_{\mathcal R}\!\circ\!
  \langle \sem{F}\!\circ\!\pi_F,\sem{G}\!\circ\!\pi_G\rangle$
  &
  $\displaystyle
  \sem{F}\,\otimes_{\mathcal R}\,\sem{G}$ \\[4pt]

  $\displaystyle 
  \sem{\exists x{:}s\,F}$ &
  $\displaystyle 
  \mathrm{aggr}^\exists_{V_{F\setminus x:s},\mathcal I(s)}\!(\sem{F})$
  &
  $\displaystyle
  \mathrm{aggr}^\exists_{\mathcal I(s)}(\lambda a. \sem{F}_{\nu[x\mapsto a]})$ \\[4pt]
  
  $\displaystyle 
  \sem{\forall x{:}s\,F}$ &
  $\displaystyle 
  \mathrm{aggr}^\forall_{V_{F\setminus x:s},\mathcal I(s)}\!(\sem{F})$
  &
  $\displaystyle
  \mathrm{aggr}^\forall_{\mathcal I(s)}(\lambda a. \sem{F}_{\nu[x\mapsto a]})$ \\[4pt]
  
  $\displaystyle 
  \sem{[x := m(\vec{T})]F}$ &
  $\displaystyle 
  \sem{F}^{*}\!\circ\!\mathcal S\!\circ
  \bigl\langle
  \pi_{V_{F\setminus x:s}}, \mathcal I(m)\circ\langle \sem{\vec{T}}_i\!\circ\!\pi_i\rangle_i
  \bigr\rangle$
  &
  $\displaystyle
  \doTerm{a}{\mathcal{I}(m)(\sem{\vec{T}})}{\sem{F}_{\nu[x\mapsto a]}}$\\
  \bottomrule
  \end{tabular}
  \end{table}
  
  \end{definition}
  
  \begin{remark}
    We define $\mathcal{V}_T:=\prod_{x:t \in \Gamma_T} \mathcal{I}(t)$. Here $\Gamma_T$ is the context of T and $s_T$ is the (unique) sort of the term $T$.  Analogously $\mathcal{V}_F:=\prod_{x:t \in \Gamma_F} \mathcal{I}(t)$. Note that if $T_1$ is a subterm of $T_2$, there is a projection $\pi_{T_1,T_2}: \mathcal{V}_{T_2} \to \mathcal{V}_{T_1}$, and analogously for formulas.
   $\vec{T}$ stands for $T_1,\ldots,T_n$.
  Moreover, $\langle\sem{\vec{T}}_i\!\circ\!\pi_i\rangle_i=\langle
          \sem{T_1}\!\circ\!\pi_1,\dots,
          \sem{T_n}\!\circ\!\pi_n
        \rangle$ and $\sem{\vec{T}}=(\sem{T_1},\ldots,\sem{T_n})$.
  The categorical semantics ensures that all involved and resulting functions 
  are morphisms in $\mathcal{C}$, i.e.\ are measurable in case that $\mathcal{C}=\Meas$, etc. With a purely set-theoretic semantics, we would need to prove measurability (or other properties) separately for each NeSy framework.

  \noindent
  That said, besides the general categorical case, \footnote{Logician's note: We don't differentiate properly between additive and multiplicative connectives/units/quantifiers for an easier presentation coherent with the NeSy literature. However, this could easily be adapted to obtain something like Girard's linear logic \citep{girardLinearLogicIts1995}, following the Zeitgeist of his transcendental syntax.} 
  for better understandability, we also translate the equations to their meaning in the category of sets.
  We work with variable valuations $\nu\in \mathcal{V}_T$ (and $\nu\in \mathcal{V}_F$), noting that elements of $\prod_{x:t \in \Gamma_T} \mathcal{I}(t)$ map variables $x:t$ to values in $\mathcal{I}(t)$. We write $\sem{T}_{\mathcal I,\nu}=\sem{T}_{\mathcal I}(\nu)$ and $\sem{F}_{\mathcal I,\nu}=\sem{F}_{\mathcal I}(\nu)$. That said, we mostly omit $\mathcal I$ and $\nu$ if clear from the context.
  \end{remark}

\subsection{Probabilistic Semantics}
\label{subsec:cont-prob-sem}

\begin{definition}
A \textbf{probabilistic NeSy framework} is a tuple $({\mathcal T}, \mathcal R)$ with ${\mathcal T} = \mathcal{M}$ the Giry monad (Ex.~\ref{ex:giry-monad}) on the category $\mathbf{Measr}$ of measure spaces and $\mathcal{R}$ a suitable 2Mon-BLat, typically the Product algebra on $[0,1]$ (Table~\ref{tab:all-algebras}, row ``Product''; see Sec.~\ref{sec:algebraic-overview} for the full overview of algebras). The integral in the computational formula rule below is the continuous analogue of the weighted sum in the distributional semantics (Sec.~\ref{subsec:dist-sem}) and arises from the Kleisli extension of $\mathcal{M}$.
\end{definition}

\noindent
The interpretation of a function $f$ of arity $n$ is a \emph{Markov kernel} (App.~\ref{sec:glossary}), which is a measurable map \(X \xrightarrow{q} \mathcal{M}(Y)\) where $\mathcal{M}$ denotes the \emph{measure monad} on the category $\mathbf{Measr}$ of measure spaces (Appendix~\ref{sec:cat-gen}). 

Our definition of a probabilistic semantics largely follows that in the original ULLER paper \citep{vankriekenULLERUnifiedLanguage2024}. A central design decision of ULLER is the use of first-order interpretations and the use of probability distributions to interpret computational function symbols. This means that ULLER (and therefore also NeSyCat) is (like Logic Tensor Networks) not built on weighted model counting, i.e.\ on probability distributions over the set of interpretations. That said, it is still possible to capture certain aspects of weighted model counting in ULLER and NeSyCat, as we will see in section~\ref{subsec:wmc} below.

Connectives in the probabilistic semantics of ULLER are interpreted assuming independence of probabilities for atomic formulas. Hence, our distributional semantics can be seen as a special case of a fuzzy semantics, where the t-norm is the probabilistic product and the t-conorm is the probabilistic sum. This means that we can use the same equations as in the fuzzy semantics (and as in LTNs), but with a different motivation. Moreover, this explains why there is no essential difference between these probabilistic and fuzzy semantics.

Let us derive from our general semantic in definition~\ref{def:Tarskian-semantics} the interpretation of monadic formulas in probabilistic semantics. For the probabilistic NeSy framework, we define the aggregation morphisms $\mathrm{aggr}_{B,C}^\forall, \mathrm{aggr}_{B,C}^\exists$ required by Def.~\ref{def:2Mon-BLat-cat} as follows: for any measure spaces $B$ and $C$, 
\begin{align*}
\mathrm{aggr}_{B,C}^\forall(f) &:= \exp \circ \, \expec_{c \sim \mu_C} [\ln \circ f(\cdot, c)] \\
\mathrm{aggr}_{B,C}^\exists(f) &:= 1 - \exp \circ \, \expec_{c \sim \mu_C} [1 - \ln \circ f(\cdot, c)]
\end{align*}
where $f: B \times C \to [0,1]$ and $\mu_C$ is the measure on $C$. In the following examples, we provide implicit definitions of these aggregation operations through their concrete realisations. In set-theoretic notation, the semantics of computational formulas can be written as, where we use the notation $\rho_m(\cdot|\vec{T}) := \mathcal{I}(m)(\sem{\vec{T}})$:
\begin{align*}
\sem{[x := m(\vec{T})]F} &= \doTerm{a}{\mathcal{I}(m)(\sem{\vec{T}})}{\sem{F}_{\nu[x\mapsto a]}}  \\
&= \int_{a \in \mathcal{I}(s_m)} \sem{F}_{\nu[x \mapsto a]} \, d \mathcal{I}(m)(\sem{\vec{T}})(a)  \\
&= \int_{a \in \mathcal{I}(s_m)} \sem{F}_{\nu[x \mapsto a]} \, d \rho_m(a|\vec{T}) \\
&= \expec_{a \sim \rho_m( \cdot\mid \vec{T})} \bigl[\sem{F}_{\nu[x\mapsto a]}\bigr] \\
&= \sum_{a \in \mathcal{I}(s_m)} \sem{F}_{\nu[x \mapsto a]} \cdot \rho_m(a|\vec{T}) \quad \text{(if $\mathcal{I}(s_m)$ is finite)}
\end{align*}

\noindent
We evaluate in the Product Algebra to obtain, where $\mu_s$ is the measure given by the measure space of $\mathcal I(s)$
\begin{align}
\sem{[x:=m(\vec{T})]F}
    &= \expec_{a \sim \rho_m( \cdot\mid \vec{T})} \bigl[\sem{F}_{\nu[x\mapsto a]}\bigr] \label{eq:monad-integral}   \\[8pt]
    \sem{\exists x{:}s\;F}
    &= 1-\exp \bigl( \expec_{a \sim \mu_s} \bigl[ \ln \bigl(1- \sem{F}_{\nu[x\mapsto a]}\bigr) \bigr] \bigr) \label{eq:prob-exists-quant} \\
    \sem{\forall x{:}s\;F}
    &= \exp \expec_{a \sim \mu_s} \bigl[ \ln \sem{F}_{\nu[x\mapsto a]} \bigr] \label{eq:prob-forall-quant}
    \\
  \sem{F \| G}
    &= \sem{F}  + \sem{G} -  \sem{F}\cdot \sem{G}, \quad
  \sem{F \& G}
    = \sem{F} \cdot \sem{G},
    \\[4pt]
  \sem{F \rightarrow G}
    &= 
          \max\bigl(1, \sem{G}\slash\sem{F} \bigr),
    \quad
  \sem{\neg F}
    = 1 - \sem{F}
    \\[2pt]
  \sem{\bot}
    &= 0, \quad
  \sem{\top}
    = 1.
\end{align}

\paragraph{$\forall$ as weighted products (finite case).} \label{par:dependence}
Let $\mathcal{I}(s)$ be finite and define the random variable $X(a):=\sem{F}_{\nu[x\mapsto a]}$.
The infinitary probabilistic $\forall$ (Eq.~\eqref{eq:prob-forall-quant}) evaluated
on $F$ with a weighted counting measure on $X$ yields
\[
  \sem{\forall x{:}s\;F}
   \;=\; \exp\Bigl( \sum_{a\in \mathcal{I}(s)} w_a  \ln X(a) \Bigr)
   \;=\; \Bigl( \prod_{a\in \mathcal{I}(s)} \exp(\ln (X(a)^{w_a})) \Bigr)
   \;=\; \prod_{a\in \mathcal{I}(s)} X(a)^{w_a}.
\]
By contrast, the finite product aggregator of the original ULLER semantics is obtained by setting $w_a = 1$ for all $a \in \mathcal{I}(s)$:
\[
  \sem{\forall x{:}s\;F} \;=\; \prod_{a\in \mathcal{I}(s)} X(a)^1 .
\]
\noindent
However, our definition of the $\forall$ quantifier is more general even in the finite case, as it allows for meaningful examples like the weighted counting measure on $\mathcal{I}(s)$ yielding the geometric mean with $n:= \# \mathcal{I}(s)$, the number of elements in $\mathcal{I}(s)$, yielding
\[
  \sem{\forall x{:}s\;F} \;=\; \prod_{i=1}^{n} X(a_i)^{1/n}.
\]
\noindent
In the case above of the geometric mean and also in many other (continuous) cases, one actually does choose only \emph{one} universal quantifier to be constant on all sorts, as for the existential quantifier. In this case one can actually work within a simpler framework using the \emph{normal} Giry monad on the category $\mathbf{Meas}$ of measurable spaces:

\begin{definition}
  A \textbf{simple probabilistic NeSy framework} is a tuple $({\mathcal T}, \mathcal R)$ with ${\mathcal T} = \mathcal{M}$ the Giry monad on the category $\mathbf{Meas}$ of measurable spaces and $\Omega$ the truth basis, normally $\Omega = \{0,1\}$.
\end{definition}

\noindent
Now, if one works in standard Borel spaces\footnote{Standard Borel spaces are measurable spaces where the underlying set is isomorphic to $\mathbb{R}$ or is finite or countable. All finite powers of the real line and all intervals, as measurable spaces, are isomorphic to the real line, as explained at \url{https://ncatlab.org/nlab/show/standard+Borel+space}. Also a measure space (or probability space) is called a standard Borel measure space ($\mathbf{BorelMeas}$) (standard Borel probability space ($\mathbf{BorelProb}$)) if and only if its underlying measurable space is standard Borel as above.} (as one most often does in practice), one obtains an even more practical version of this simple probabilistic NeSy framework. This is very useful in practice since one can see an uncountable standard Borel space just as a set that is canonically equipped with the Borel sigma-algebra and Lebesgue measure. Also a finite or countable standard Borel space can be seen as a set that is canonically equipped with the discrete sigma-algebra and counting measure.

\begin{definition}
  A \textbf{simple standard Borel NeSy framework} is a tuple $({\mathcal T}, \mathcal R)$ with ${\mathcal T} = \mathcal{M}$ the Giry monad on the category $\mathbf{BorelMeas}$ of standard Borel spaces and $\Omega$ the truth basis, normally $\Omega = \{0,1\}$.
\end{definition}
\noindent
Additionally, if one chooses the measures for universal and existential quantification to be given by a density functions $f$ (most often the case in practice), one obtains an implementation-friendly Lebesgue-probabilistic NeSy semantics for the quantifiers $\forall$ and $\exists$:
\begin{align}
  \sem{\forall x{:}s\;F} &= \exp \int_{a \in \mathbb \mathcal{I}(s)}  \ln \sem{F}_{\nu[x\mapsto a]} \, f(a) \, da \label{eq:prob-forall-quant-density} \\
  \sem{\exists x{:}s\;F} &= 1 - \exp  \int_{a \in \mathbb \mathcal{I}(s)} \ln (1 - \sem{F}_{\nu[x\mapsto a]}) \, f(a) \, da \label{eq:prob-exists-quant-density}
\end{align}
\noindent
In the same spirit, if the semantics of the monadic formula~\ref{eq:monad-integral} should only use probability measures admitting a density w.r.t. the Lebesgue measure\footnote{This is hidden in the notation of the semantics of the monadic formula in the original ULLER semantics in \citep{vankriekenULLERUnifiedLanguage2024}, Eq. 20; although there the density is mistakenly confused with the corresponding measure.}, one can work with a simplified Giry monad on the category $\mathbf{BorelMeas}$, sending a standard Borel space to the standard Borel space of probability measures on that space admitting a density w.r.t. the Lebesgue measure. Combined with the quantifiers defined w.r.t Lebesgue densities from eq.~\eqref{eq:prob-forall-quant-density} and eq.~\eqref{eq:prob-exists-quant-density}, this makes the implementation simpler and more efficient and is the default choice in practice anyway.

Actually, depending on the use case, one might want to use a probability monad on one of the following categories\footnote{Which are by far not all, see \url{https://ncatlab.org/nlab/show/monads+of+probability,+measures,+and+valuations} for more examples.}, listed in increasing complexity in Table~\ref{tab:prob-cats} summarising typical suitability.

\begin{table}[h]
\centering
\caption{Suitability of base categories and their probability monads (see Appendix~\ref{sec:cat-gen} for definitions of these categories and the distinction between them)}
\label{tab:prob-cats}
\small
\begin{tabular}{@{}llccccccc@{}}
\toprule
\textbf{Category} & \textbf{Monad} & \textbf{FS} & \textbf{Finite} & \textbf{$\cong \mathbb Z$} & \textbf{\(\cong \mathbb R\)} & \textbf{+DQ} & \textbf{+DPQ} & \textbf{+HO} \\ \midrule
\(\mathbf{Set}\)          & $\mathcal{D}$ (Ex.~\ref{ex:dist-monads})        & $\checkmark$ & $\checkmark$ & $\checkmark$ & $\checkmark$ & $\times$ & $\times$ & $\checkmark$ \\
\(\mathbf{BorelMeas}\)    & $\mathcal{G}|_\mathrm{Borel}$                    & $\circ$ & $\circ$ & $\checkmark$ & $\checkmark$ & $\times$ & $\times$ & $\times$ \\
\(\mathbf{BorelMeasr}\)   & $\mathcal{M}|_\mathrm{Borel}$                    & $\circ$ & $\circ$ & $\circ$ & $\circ$ & $\checkmark$ & $\circ$ & $\times$ \\
\(\mathbf{BorelProb}\)    & $\mathcal{O}|_\mathrm{Borel}$                    & $\circ$ & $\circ$ & $\circ$ & $\circ$ & $\circ$ & $\checkmark$ & $\times$ \\
\(\mathbf{Meas}\)         & $\mathcal{G}$ (Ex.~\ref{ex:giry-monad})          & $\circ$ & $\circ$ & $\checkmark$ & $\checkmark$ & $\times$ & $\times$ & $\times$ \\
\(\mathbf{Measr}\)        & $\mathcal{M}$ (Prop.~\ref{prop:measure-monad})   & $\circ$ & $\circ$ & $\circ$ & $\circ$ & $\checkmark$ & $\circ$ & $\times$ \\
\(\mathbf{Prob}\)         & $\mathcal{O}$ (Prop.~\ref{prop:measure-monad})   & $\circ$ & $\circ$ & $\circ$ & $\circ$ & $\circ$ & $\checkmark$ & $\times$ \\
\(\mathbf{QBS}\)          & $P$ (Def.~\ref{def:qbs})                         & $\circ$ & $\circ$ & $\checkmark$ & $\checkmark$ & $\times$ & $\times$ & $\checkmark$ \\
\bottomrule
\end{tabular}
\end{table}

\noindent
The symbols $\checkmark$, $\circ$, and $\times$ indicate whether a given category is \emph{the simplest adequate choice} ($\checkmark$), \emph{technically compatible but unnecessarily complex or not the canonical setting} ($\circ$), or \emph{incompatible} ($\times$) for probability measures of the kind indicated by the column. For example, finitely supported distributions can be defined on $\mathbf{Meas}$, but the measurable structure adds nothing over plain $\mathbf{Set}$, hence the $\circ$; conversely, sort-dependent quantifiers require a measure on each sort, which is only available in $\mathbf{Measr}$ or $\mathbf{Prob}$ (not in $\mathbf{Meas}$ or $\mathbf{Set}$), hence the $\times$ in those cells. The Borel variants restrict to standard Borel spaces, which simplifies measure theory at the cost of excluding exotic spaces.

The column abbreviations are: ``FS'' = finitely supported probability measures, ``Finite'' = probability measures on finite spaces, ``$\cong \mathbb Z$'' = on countable (discrete) spaces, ``$\cong \mathbb R$'' = on uncountable (continuous) spaces, ``+DQ'' = sort-dependent quantifiers via measures, ``+DPQ'' = sort-dependent quantifiers via probability measures, and ``+HO'' = higher-order logic (requires a Cartesian closed category). Notably, $\mathbf{QBS}$ (quasi-Borel spaces, Def.~\ref{def:qbs}) \citep{heunenConvenientCategoryHigherOrder2017} is the only category in the table that supports both continuous distributions and higher-order logic, since it is Cartesian closed and carries a probability monad that agrees with the Giry monad on standard Borel spaces. This has made $\mathbf{QBS}$ the preferred semantic foundation for higher-order probabilistic programming languages \citep{staton2016semantics,vakar2019domain,scibior2018denotational}; see Section~\ref{subsec:monadic-pp} for details.

\subsection{Infinitary $\text{LTN}_p$ Semantics}
\label{subsec:inf-LTN}

\begin{definition}
  A \textbf{LTN-like NeSy framework} is a tuple $({\mathcal T}, \mathcal R)$ with ${\mathcal T} = \mathcal{O}$ the probability monad on the category $\mathbf{Prob}$ of probability spaces and $\Omega$ the truth basis, normally $\Omega = \{0,1\}$, and $\mathcal{R}$ a suitable 2Mon-BLat, typically the S-Product algebra (Table~\ref{tab:all-algebras}, row ``S-Prod.'') as in \citep{DBLP:journals/ai/BadreddineGSS22}. This is the infinitary generalisation of the finitary LTN semantics from Sec.~\ref{par:LTN}.
  \end{definition}

\paragraph{Setting} 
Stable product real logic of Logic Tensor Networks \citep{DBLP:journals/ai/BadreddineGSS22} uses $p$-means for finite quantification. 
However, since domains are generally infinite, we also need to aggregate infinite many truth–scores
\((x_i)_{i\in I}\subseteq[0,1]\).
The power–mean extends from the finite case
to an \emph{integral} form that is well defined whenever the data are
\(L^{p}\)-integrable. Let \((X,\mathcal A,\rho)\) be a probability space and 
\(f:X\!\to[0,1]\subseteq\mathbb R\) a measurable map with 
\(\displaystyle\int_X f\,d\rho\le 1\).
Because \(0\le f\le 1\), one automatically has 
\(f\in L^{p}(\rho)\) for every real \(p\), so $p$‑means are always defined.
This bounded–by–one assumption reflects the fact that in our logical
reading a truth‑score never exceeds~\(1\).

\paragraph{Integral $p$-Mean Quantifiers}
The idea behind $p$-mean aggregation is that the hyperparameter~$p$ acts as a \emph{strictness dial}: for small~$p$ the quantifier behaves like an average, tolerating a few low truth-scores among many high ones, while for $p\to\infty$ it converges to the supremum (for~$\exists$) or infimum (for~$\forall$), enforcing a strictly logical reading. In a learning setting one typically anneals~$p$ upward during training, starting with a smooth, gradient-friendly landscape and gradually tightening toward classical logic.

The infinitary $\text{LTN}_p$ semantics is a modification of the probabilistic semantics, which is motivated by replacing the quantifiers in equation~\eqref{eq:prob-exists-quant} and~\eqref{eq:prob-forall-quant}. We generalise this: by working in the category $\category{Prob}$, for any sort $s$ we have to provide a probability measure $\rho_s$ on $\mathcal{I}(s)$. This enables us to obtain $p$-means for infinite\footnote{Logic Tensor Networks \citep{DBLP:journals/ai/BadreddineGSS22} were originally designed for finite domains only: all variables are mapped to finite sequences of real numbers, bypassing the notion of random variables and probability distributions entirely. \cite{slusarzLogicDifferentiableLogics2023} criticise this limitation, arguing that a differentiable logic must account for the unknown probability distribution from which the training data is sampled (see \citep[Sec.~4]{slusarzLogicDifferentiableLogics2023}). While this criticism is justified for the original LTN formulation, the infinitary $\text{LTN}_p$ semantics we define here shows that LTN's $p$-norm quantifiers \emph{can} be extended to infinite (and continuous) domains by replacing finite sums with integrals against a probability measure $\rho_s$. This extension is our contribution, not part of the original LTN.} domains:

\begin{equation}\label{eq:p-mean}
M_p(a_1,\dots,a_n):=\displaystyle\Bigl(\frac{1}{n}\sum_{i=1}^{n} a_i^{\,p}\Bigr)^{\!1/p}, \quad   
M_{p}(f;\rho_s)
    :=\Bigl(\,\int_{x \in X} f(x)^{p}\,d\rho_s(x)\Bigr)^{\!1/p},
\end{equation}

\noindent
and these extend to $M_0(a_1,\dots,a_n):=\Bigl(\prod_{i=1}^{n} a_i\Bigr)^{\frac{1}{n}}$ and
\(M_{0}(f;\rho_s):=\exp \bigl(\int \ln f\,d\rho_s\bigr)\). For \(p\to\infty\) we recover the supremum.  Take \(X=\{1,\dots,N\}\) with counting measure \(1/N\),
then \( M_{p}(f;\rho_s)\) reduces to $M_p(a_1,\dots,a_n)$ or choose weights \(w_i\) summing up to 1 for \(i=1,\dots,N\) to obtain the weighted $p$-mean. The aggregated 2Mon-BLat is similar to the probabilistic one, except for the Reichenbach implication as implication and the following aggregation functions. For a hyperparameter $1 \leq p< \infty$ of $\text{LTN}_p$, let $\mathrm{aggr}^\exists_{\mathcal I(s)}(f) := M_{p}(f;\rho_s)$ and $\mathrm{aggr}^\forall_{\mathcal I(s)}(f) := 1-M_{p}(\lambda x. f(1-x);\rho_s)$. As a result, equation~\eqref{eq:p-mean} now becomes:
\begin{align*}
  \sem{\exists x{:}s\;F}
    &
    =\Bigl(\,\int_{a\in\mathcal{I}(s)} \bigl(\sem{F}_{\nu[x\mapsto a]}\bigr)^{p}\,d\rho_s(a)\Bigr)^{\!1/p}.
\end{align*}

\noindent
As in \cite{DBLP:journals/ai/BadreddineGSS22}, this is for $1\leq p < \infty$, where for $p \to \infty$ we recover the supremum.

\paragraph{Log-Power $q$-Mean Quantifiers}
The $p$-mean quantifiers above are tied to the Łukasiewicz or Reichenbach style of combining truth values additively. A different, arguably more natural, choice for a product-based real logic is to aggregate \emph{multiplicatively}: universal quantification should behave like an infinite conjunction under the (Goguen) product t-norm, i.e.\ like a (possibly weighted) geometric mean. The logarithm converts this product into a sum, so the natural ``$p$-mean in log-space'' is a power mean applied to $\ln\sem{F}$ rather than to $\sem{F}$ itself. This motivates the following alternative pair of quantifiers with $1/2 \leq q \leq 1$, where for $q \to 1$ the universal quantifier converges to the product mean (geometric mean), while the existential quantifier converges to the probabilistic sum:
\begin{align*}
  \sem{\forall x{:}s\;F}
    &
    :=\exp\Bigl(\Bigl(\,\int_{a \in \mathcal{I}(s)} \bigl(\ln \sem{F}_{\nu[x\mapsto a]}\bigr)^{q}\,d\rho_s(a)\Bigr)^{\!1/q}\Bigr), \quad
  \sem{\exists x{:}s\;F} := \sem{\neg \forall x{:}s\;\neg F}.
\end{align*}

A subtle but important design choice concerns the probability measure used inside the integrals: should it be attached to the \emph{sort}~$s$ (i.e.\ to the domain of quantification) or to the individual \emph{variable}~$x$? We opt for the former.
Hence, our probability measure $\rho_s$ depend on the sort $s$ of the variable $x$ in the quantifier, since it is given by the probability measure of the probability space of $\mathcal I(s)$. This stands in contrast to \cite{slusarzLogicDifferentiableLogics2023}, where the probability measure depends directly on the variable $x$.

\paragraph{Lebesgue-Density Variants}
In many practical settings the domain $\mathcal I(s)$ is a subset of $\mathbb R^n$ and the probability measure $\rho_s$ admits a density with respect to the Lebesgue measure. Working with a density function~$f$ directly is often more convenient for implementation: one replaces $d\rho_s$ by $f(a)\,da$ and can leverage standard numerical integration or Monte-Carlo sampling.
Just as in the case of the simple probabilistic NeSy framework, one can also work with a simple Lebesgue NeSy framework, and choose the measures for universal and existential quantification to be given by a Lebesgue-density function $f$ to obtain. for the case of log-power $q$-mean quantifiers:

\begin{align*}
  \sem{\forall x{:}s\;F}
    &
    :=\exp\Bigl(\Bigl(\,\int_{a \in \mathcal{I}(s)} \bigl(\ln \sem{F}_{\nu[x\mapsto a]}\bigr)^{q} f(a) \,da\Bigr)^{\!1/q}\Bigr), \quad
  \sem{\exists x{:}s\;F} := \sem{\neg \forall x{:}s\;\neg F},
\end{align*}

\noindent
or in the case of the integral $p$-mean quantifiers originally used by LTNs:

\begin{align*}
  \sem{\exists x{:}s\;F}
    &
    =\Bigl(\,\int_{a\in\mathcal{I}(s)} \bigl(\sem{F}_{\nu[x\mapsto a]}\bigr)^{p} f(a) \,da\Bigr)^{\!1/p}, \quad
  \sem{\forall x{:}s\;F} := \sem{\neg \exists x{:}s\;\neg F}.
\end{align*}

\subsection{Infinitary STL Semantics}\label{subsec:stl-sem}

\noindent
Signal Temporal Logic (STL) is a temporal logic for expressing properties of signals. STL is particularly useful for modelling and analysing the temporal aspects of real-time systems, such as the timing and sequencing of events. 

In the semantics of the original STL, we do not have any implication connective, nor neutral elements. We still need to interpret the syntactic implication connective as some form of semantical implication and the syntactic $\bot$ and $\top$ as $-\infty$ and $\infty$, the latter as in \cite{slusarzLogicDifferentiableLogics2023}. Also keep in mind, that this does not touch the truth designations of  $-\infty$ as absolute falsity and $\infty$ as absolute truth. That is, STL works with these degrees of truth and falsity, which could be taken together to form one single degree (of probability), and which the absolute degrees are infinite. Therefore we model STL with a truth basis containing only one element: $\{ 1\}$, a basis element which is then scaled by to form the extended real numbers as truth space. A consequence of this that there are no meaningful non-computational predicates in STL.

What also can not be ignored is that STL does not directly uses a 2Mon-BLat, but only approximates one. It works within the normal extended real numbers algebra $(\tilde{\mathbb{R}}, \max, \min, +, *)$ and then goes on to approximate the $\min$ and $\max$ operations. The are many different ways to do this, but one of the most recent ones is to use the $\mathsf{A}^r$ and $\mathsf{O}^r$ operators as defined in \cite{varnaiRobustnessMetricsLearning2020}. Additionally, STL is not concerned with the operations $\otimes$ and $\oplus$ of the 2Mon-BLat, these are not used in the semantics of STL, and are just kept to be the standard operations $+$, $*$ of the extended real numbers algebra.

For these reasons, in order to faithfully model STL, in a way that makes it comparable to other semantics, we would need to extend our syntax and semantics, and we would also need to allow to approximate 2Mon-BLats. This however, is out of scope for this paper, and will be discussed in a future work, and yet we still give a first sketch:
\begin{align}
  \sem{[x:=m(\vec{T})]F}
  &:= \expec_{a \sim \mu_m( \cdot\mid \vec{T})} \bigl[\sem{F}_{\nu[x\mapsto a]}\bigr] \label{eq:STL-monad-integral}   \\
    \sem{\exists x{:}s\;F} 
    &:= \mathsf{O}^r_{a \in \mathcal{I}(s)} \bigl(\sem{F}_{\nu[x\mapsto a]}\bigr), \quad
    \sem{\forall x{:}s\;F}
    := \mathsf{A}^r_{a \in \mathcal{I}(s)} \bigl(\sem{F}_{\nu[x\mapsto a]}\bigr),
    \\[4pt]
    \sem{F \| G}
      &:= \mathsf{O}^r(\sem{F},\sem{G}),
  \quad
    \sem{F \& G}
      := \mathsf{A}^r(\sem{F},\sem{G}) \label{eq:stl-and}
      \\[4pt]
    \sem{F \rightarrow G}
      &:= 
      \sem{\neg F \| G} =  \mathsf{O}(- \sem{F},\sem{G}),
      \quad
    \sem{\neg F}
      := - \sem{F}
      \\[2pt]
    \sem{\bot}
      &:= -\infty, \quad
    \sem{\top}
      := \infty.
  \end{align}

\noindent
The STL robustness metrics are defined as in \cite{slusarzLogicDifferentiableLogics2023} and originally in \cite{varnaiRobustnessMetricsLearning2020}:
\[
\mathsf{A}^r_{a \in M}(a) = \begin{cases}
\frac{\sum_{a} a_{min} e^{\tilde{a}} e^{r \tilde{a}}}{\sum_{a} e^{r \tilde{a}}} & \text{if } a_{min} < 0 \\[0.5em]
\frac{\sum_{a} a e^{-r \tilde{a}}}{\sum_{a} e^{-r \tilde{a}}} & \text{if } a_{min} > 0 \\[0.5em]
0 & \text{if } a_{min} = 0
\end{cases}
\]
where $r \in \mathbb{R}^+$ (constant), $a_{min} = \min_{a \in M}(a)$, and $\tilde{a} = \frac{a - a_{min}}{a_{min}}$. $\mathsf{A}^r$ is an approximation of the min operation, and for $r \to \infty$ it converges to it. Therefore, its notation is similar to the notation of the min operation, with $\mathsf{A}^r_{b \in N}(f(b)) := \mathsf{A}^r(\mathrm{im}(f)) := \mathsf{A}^r_{a \in \mathrm{im}(f)}(a)$. The operator $\mathsf{O}^r_{a \in M}$ is defined as $-\mathsf{A}^r_{a \in M}(-a)$.\footnote{See \cite{varnaiRobustnessMetricsLearning2020} for details.} For infinite domains, the minimum is replaced by the infimum $\inf_{a \in M}(a)$, and the summations $\sum_{a \in M}$ are replaced by integrals $\int_{a \in M}d\mu_s(a)$, where $\mu_s$ is the measure given by the measure space of $\mathcal I(s)$.

\subsection{Weighted Model Counting and Weighted Model Integration}\label{subsec:wmc}

The purpose of this subsection is not to introduce new semantics, but to demonstrate that NeSyCat's monadic do-notation is expressive enough to recover weighted model counting (WMC) and weighted model integration (WMI) as special cases. This shows that the move to monadic semantics does not lose any expressiveness compared to the original ULLER formulation; rather, the do-notation provides a uniform syntax for iterated monadic binding that naturally captures the nested summations (or integrations) characteristic of WMC/WMI.

ULLER can model certain aspects of WMC in a probabilistic semantics.
However, instead of summing up literal or model weights, one needs to sum up weights of variable valuations. In the case of ULLER \citep{vankriekenULLERUnifiedLanguage2024}, we have the following definition.
Given an interpretation $\mathcal{I}$ and a formula $F$ that is classical (i.e.\ without computational  symbols) with context $\Gamma_F := \{x_1:s_1, \ldots, x_n:s_n\}$, the domains of the variables are given by $\mathcal{I}(s_1), \ldots ,\mathcal{I}(s_n)$\footnote{In the original ULLER paper these were written as $\Omega_1 := \mathcal{I}(s_1), \ldots ,\Omega_n := \mathcal{I}(s_n)$.}. This yields the weighted model count (WMC) as follows:

\begin{align*}
\mathrm{WMC}(F,w)
&= \sum_{\vec{a} \in \mathcal{I}(s_1) \times \cdots \times \mathcal{I}(s_n)} w(\vec{a}) \llbracket F \rrbracket_{\nu{[x_1 \mapsto a_1,\,\ldots,\,x_n \mapsto a_n]}} \\
&= 
\sum_{a_1 \in \mathcal{I}(s_1)}\;
  \cdots\;
  \sum_{a_n \in \mathcal{I}(s_n)}
  \;
  w(a_1,\dots,a_n)\,
  \llbracket F \rrbracket_{\nu{[x_1 \mapsto a_1,\,\ldots,\,x_n \mapsto a_n]}} 
\end{align*}

\noindent
If the weight function factorises, i.e. the random variables
\(x_1,\dots,x_n\) are assumed \textbf{independent}\footnote{As in \citep[p.~16]{vankriekenULLERUnifiedLanguage2024}.} —then for every assignment
\((a_1,\dots,a_n)\in \mathcal{I}(s_1)\times\cdots\times\mathcal{I}(s_n)\):
\[
  w(a_1,\dots,a_n)
  \;=\;
  \prod_{i=1}^{n} \rho_{f_i}(a_i).
\]

\noindent
Consequently, the weighted model count becomes
\[
  \mathrm{WMC}(F, f_1,\dots,f_n)
  \;=\;
  \sum_{a_1\in\mathcal{I}(s_1)}\;
  \cdots\;
  \sum_{a_n\in\mathcal{I}(s_n)}
  \biggl(
    \prod_{i=1}^{n} \rho_{f_i}(a_i)
  \biggr)
  \,
  \llbracket
    F
  \rrbracket_{\nu{[x_1\mapsto a_1,\ldots,x_n\mapsto a_n]}}.
\]

\noindent
This WMC can be expressed in the language of NeSy systems as follows \citep[p.~234]{vankriekenULLERUnifiedLanguage2024}:
\[ [x_1 := f_1 (), \ldots , x_n := f_n ()]F\]

\noindent
In the linearly \textbf{dependent} case, rewrite it via the chain rule,
\[
  w(a_1,\dots,a_n)
  \;=\;
  \prod_{i=1}^{n}
    \rho_{f_i}\!\bigl(a_i \mid
               a_1,\dots,a_{i-1}\bigr),
\]
to make the conditional dependencies explicit. In this case, the WMC becomes
\[
  \sum_{a_1\in\mathcal{I}(s_1)}\;
  \cdots\;
  \sum_{a_n\in\mathcal{I}(s_n)}
  \biggl(
    \prod_{i=1}^{n}
    \rho_{f_i}\!\bigl(a_i \mid
               a_1,\dots,a_{i-1}\bigr)
  \biggr)
  \,
  \llbracket
    F
  \rrbracket_{\nu{[x_1\mapsto a_1,\ldots,x_n\mapsto a_n]}}.
  \]
\noindent
In the \textbf{continuous} case we obtain weighted model integration\footnote{Compare with \citep{morettinEfficientWeightedModel2017}.} (WMI) as follows:

\begin{align*}
\int_{a_1\in\mathcal{I}(s_1)}\!\cdots\!
  \int_{a_n\in\mathcal{I}(s_n)}
  \!\sem{F}_{\nu{[x_1\mapsto a_1,\ldots,x_n\mapsto a_n]}}d \rho_{f_n}\!\bigl(a_n \mid
  a_1,\dots,a_{n-1}\bigr)\cdots d\rho_{f_1}\!\bigl(a_1\bigr)
\end{align*}

\noindent
Finally, we can also express even more general dependencies than linear ones. Given any Bayesian network with a set of variables $x_1,\ldots,x_n$, we can express this in ULLER as follows:
\[
  \sum_{a_1\in\mathcal{I}(s_1)}\;
  \cdots\;
  \sum_{a_n\in\mathcal{I}(s_n)}
  \biggl(
    \prod_{i=1}^{n}
    \rho_{f_i}\!\bigl(a_i \mid
               \textrm{parents}(a_i)\bigr)
  \biggr)
  \,
  \llbracket
    F
  \rrbracket_{\nu{[x_1\mapsto a_1,\ldots,x_n\mapsto a_n]}}.
  \]
and in ULLER, this is expressed as (assuming that the $x_i$ are topologically ordered, i.e.\ $x_i$ can only a parent of $x_j$ if $i<j$):

\[[x_1 := f_1 (), x_2 := f_2 (\textrm{parents}(x_2)),\ldots , x_n := f_n (\textrm{parents}(x_n))]F.\]

\noindent
In each of the above cases -- independent, linearly dependent, Bayesian, and continuous -- the nested sums or integrals arise from iterating the monadic do-notation $[x_i := f_i(\cdot)]$, which unfolds via the Kleisli extension of the distribution or Giry monad. The contribution of NeSyCat here is not a new derivation of WMC/WMI, but rather the demonstration that the monadic framework subsumes these constructions: the do-notation provides a single, compositional syntax from which WMC, WMI, and Bayesian factorisations all emerge as instances of iterated monadic binding.

\section{Related Work and Their Examples}\label{sec:related}
Monad-based dynamic logic \citep{DBLP:journals/fac/MossakowskiSG10} is similar to our approach, but differs in some important aspects. In \cite{DBLP:journals/fac/MossakowskiSG10}, not the whole of $\mathcal T \Omega$ is used as the space of truth values, but only a subset of it, namely the \emph{pure} computations $p$ with truth-valued result. These are discardible, i.e.\ they can be left out in a sequence of computations, and copyable, i.e.\ deterministic. The latter means that $[x:=p, y:=p]x=y$ holds. In the non-empty powerset monad, the distribution and the Giry monads, all computations are discardible, but only those in the image of $\eta$ are copyable. That is, only $T$ and $F$ (in the non-empty non-determinism monad) and $\delta_T$ and $\delta_F$ (in the distribution and Giry monads) are copyable. However, in the non-determinism monad, we clearly want $B\equiv\{T,F\}$ as a truth value, and in the 
distribution and the Giry monads, we want all probabilities $[0,1]$ as a truth values (and this space is isomorphic to probability distributions over $\{T,F\}$). Hence, we do not want a copyability assumption in our semantics. It seems that this is common in monads used for knowledge representation, as opposed to monads used for programming language semantics as in \cite{DBLP:journals/fac/MossakowskiSG10}.
Note that both views may even be useful for one and the same monad (e.g.\ the non-empty powerset monad), depending on its use. 

\subsection{Example: Weather Prediction of DeepSeaProbLog}

As an illustration of how existing neurosymbolic systems using \emph{continuous} probability distributions\footnote{Note that a proper treatment of these requires category theory, see sections~\ref{subsec:more-monads} and~\ref{sec:cat-nesy}.} can be expressed in NeSyCat, using the Giry monad on Borel measurable spaces with probailstic semantics (see section~\ref{subsec:cont-prob-sem}). Consider the following weather prediction example from DeepSeaProbLog \citep{smetNeuralProbabilisticLogic2023}. The original DeepSeaProbLog program uses neural distributional facts to model humidity detection and temperature prediction:

\begin{minted}{prolog}
humid(Data) ~ bernoulli(humid_detector(Data)).
temp(Data, T) ~ normal(temperature_predictor(Data)).

good_weather(Data) :- humid(Data) =:= 1, temp(Data) < 0.
good_weather(Data) :- humid(Data) =:= 0, temp(Data) > 15.

query(good_weather(world)).
\end{minted}

\noindent
Here, \texttt{humid\_detector} and \texttt{temperature\_predictor} output distribution parameters, while \texttt{world} represents a specific dataset. In NeSyCat syntax, this last query becomes the following formula (later denoted by ${F}$) with $\mathbf{world}:=\mathcal{I}(\mathbf{data_1})$:
$$\begin{array}{l}
\qquad  [h := \mathtt{bernoulli}(\mathtt{humid\_detector}(\mathbf{data_1}))] \\
\qquad\quad [t := \mathtt{normal}(\mathtt{temperature\_predictor}(\mathbf{data_1}))] \\
\qquad\qquad  (h = 1 \land t < 0) \lor (h = 0 \land t > 15)
\end{array}$$
where \texttt{humid\_detector} is a function returning parameters for a Bernoulli distribution, and \texttt{temperature\_predictor} returns parameters $(\mu, \sigma)$ for a normal distribution.\footnote{In \cite{smetNeuralProbabilisticLogic2023} the \texttt{humid\_detector} returns a probability distribution over $[0,1]$ and \texttt{temperature\_predictor} returns a probability distribution over $\mathbb{R}^2$, that is that they are also monadic functions, that are composed in the Kleisli category with the monadic functions \texttt{Bernoulli} and \texttt{Normal}, respectively. For simplicity of presentation, we do not explicitly mention this in the syntax and assume they are just deterministic functions.} The nested monadic assignments capture the same probabilistic dependencies as the original program, demonstrating how NeSyCat's uniform syntax can express diverse neurosymbolic paradigms. The semantic evaluation $\sem{F}_{\mathcal{I},\nu}$ of this formula yields the probability distribution over truth values, corresponding to the query result in DeepSeaProbLog.

\paragraph{DeepSeaProbLog semantics vs NeSyCat (same query).}
Given the signature 
\begin{align*}
&Worlds,\ Unit\_Interval,\ Reals^2, \ Reals \in \Sigma, \\
&\mathbf{data1} \in \mathrm{Const} \\
&humid\_detector, temperature\_predictor \in \mathrm{Func} \\
&Bernoulli, \ Normal \in \mathrm{mFunc},
\end{align*}
define the interpretation function
\begin{align*}
&\mathbf{world}:=\mathcal{I}(\mathbf{data_1}) \in Worlds:=\mathcal{I}(Worlds) \\
&\mathcal{I}(temperature\_predictor): Worlds \to \mathbb{R}^2,
&&\mathcal{I}(humid\_detector): Worlds \to [0,1];\\
&\mathcal N := \mathcal I (Normal):\mathbb{R}^2 \to \mathcal G(\mathbb{R}),
&&\mathcal B := \mathcal I (Bernoulli):[0,1] \to \mathcal G(\{0,1\}).
\end{align*}
and label
\begin{align*}
&(\mu,\sigma)=\mathcal I(\mathtt{temperature\_predictor})(\mathbf{world}),
&&p=\mathcal I(\mathtt{humid\_detector})(\mathbf{world});\\
&T\sim\mathcal N(\mu,\sigma^2),
&&H\sim\mathcal{B}(p).
\end{align*}
Then define events $A:=\{H=1,\ T<0\}$ and $B:=\{H=0,\ T>15\}$ (disjoint).
DeepSeaProbLog assigns the query probability as an expectation of an indicator under the joint measure (here product measure by independence of PCFs):
\[
  P_{\mathrm{DSP}}(\mathtt{good\_weather}) 
    = \mathbb E[\mathbf 1_A + \mathbf 1_B]
    = p\!\int_{\mathbb R} \mathbf 1(t<0)\,\varphi_{\mu,\sigma}(t)\,dt
      + (1-p)\!\int_{\mathbb R} \mathbf 1(t>15)\,\varphi_{\mu,\sigma}(t)\,dt,
\]
where $\displaystyle\varphi_{\mu,\sigma}(t):=\frac{1}{\sqrt{2\pi}\,\sigma}\exp\!\left(-\frac{(t-\mu)^2}{2\sigma^2}\right)$ is the normal density. In NeSyCat, the monadic rule for $[x:=m(\cdot)]F$ evaluates to an expectation (Eq.~\eqref{eq:monad-integral}), hence for $F$ above
\[
\begin{aligned}
  \sem{F}
    &= \mathbb E_H\!\left[\mathbb E_T\!\left[\mathbf 1_{\{H=1,T<0\}} + \mathbf 1_{\{H=0,T>15\}} \,\middle|\, H\right]\right] \\
    &= \mathbb E_H\!\left[\mathbf 1_{\{H=1\}}\,\mathbb E_T[\mathbf 1_{\{T<0\}}] + \mathbf 1_{\{H=0\}}\,\mathbb E_T[\mathbf 1_{\{T>15\}}]\right] \\
    &= \mathbb E_H[\mathbf 1_{\{H=1\}}]\!\int_{-\infty}^{0}\!\varphi_{\mu,\sigma}(t)\,dt
       + \mathbb E_H[\mathbf 1_{\{H=0\}}]\!\int_{15}^{\infty}\!\varphi_{\mu,\sigma}(t)\,dt \\
    &= p\!\int_{-\infty}^{0}\!\varphi_{\mu,\sigma}(t)\,dt + (1-p)\!\int_{15}^{\infty}\!\varphi_{\mu,\sigma}(t)\,dt \\
    &= p\,\Phi\!\left(\frac{0-\mu}{\sigma}\right) + (1-p)\,\Bigl(1-\Phi\!\left(\frac{15-\mu}{\sigma}\right)\Bigr) \\
    &= P_{\mathrm{DSP}}(\mathtt{good\_weather}).
\end{aligned}
\]
Here, $\Phi$ denotes the standard normal cumulative distribution function.
Thus, for this program our semantics coincides with DeepSeaProbLog's possible-world semantics while avoiding world enumeration.

Beyond single-instance queries, the distributional quantifiers from Table~\ref{tab:semantics} are meaningful in this setting. Three typical use cases are:

\begin{itemize}
  \item \textbf{All stations (universal aggregation).} Probability that all weather stations have good weather:
  $$\begin{array}{l}
  \forall s \in \mathtt{WeatherStations} \\
  \quad [h := \mathtt{bernoulli}(\mathtt{humid\_detector}(\mathbf{World}_s))] \\
  \qquad [t := \mathtt{normal}(\mathtt{temperature\_predictor}(\mathbf{World}_s))] \\
  \quad\qquad (h = 1 \land t < 0) \lor (h = 0 \land t > 15)
  \end{array}$$
  which evaluates via $\prod_{s}(\cdot)$ in the Product algebra.

  \item \textbf{Exists a region (existential aggregation).} Probability that at least one region has good weather:
  $$\begin{array}{l}
  \exists r \in \mathtt{Regions} \\
  \quad [h := \mathtt{bernoulli}(\mathtt{humid\_detector}(\mathbf{World}_r))] \\
  \qquad [t := \mathtt{normal}(\mathtt{temperature\_predictor}(\mathbf{World}_r))] \\
  \quad\qquad (h = 1 \land t < 0) \lor (h = 0 \land t > 15)
  \end{array}$$
  which evaluates via $1-\prod_{r}(1-\cdot)$ (probabilistic sum).

  \item \textbf{Always over time (universal over time).} Probability that the weather is good for all time slots:
  $$\begin{array}{l}
  \forall \tau \in \mathtt{TimeSlots} \\
  \quad [h := \mathtt{bernoulli}(\mathtt{humid\_detector}(\mathbf{World}_\tau))] \\
  \qquad [t := \mathtt{normal}(\mathtt{temperature\_predictor}(\mathbf{World}_\tau))] \\
  \quad\qquad (h = 1 \land t < 0) \lor (h = 0 \land t > 15)
  \end{array}$$
  again aggregating with a product over $\tau$.
\end{itemize}

\paragraph{On $\forall$-aggregation: product vs.
infimum vs. LTN.}
For distributional semantics there are several meaningful choices for
aggregating universal quantification:
\begin{itemize}
  \item \textbf{Product} (probabilistic $\forall$) as in~\eqref{eq:prob-forall-quant}:
  $\sem{\forall x:s\,F}=\prod_{a\in\mathcal I(s)} \sem{F}_{\nu[x\mapsto a]}$.
  This reads as "probability that \emph{all} independent\footnote{See paragraph~\ref{par:dependence} for a critical discussion of independence.} events hold".
  It is however extremely sensitive: a single zero (e.g., one faulty station
  reporting bad weather) collapses the product to~0 and, even without zeros,
  the value decays exponentially with the number of stations, at least for continuous values. 

  \item \textbf{Infimum/min} (classical fuzzy $\forall$):
  $\inf_{a} \sem{F}_{\nu[x\mapsto a]}$. This captures a strict
  worst-case reading and does not shrink when many stations are near~1, but it
  is still killed by a single zero and ignores the distribution of the other
  values.

  \item \textbf{LTN-style $p$-mean of complements}
  (Sec.~\ref{par:LTN}, App.~\ref{subsec:inf-LTN}):
  $\displaystyle\sem{\forall x:s\,F}=1-\Bigl(\frac{1}{|\mathcal I(s)|}\sum_{a}(1-\sem{F}_{\nu[x\mapsto a]})^{p}\Bigr)^{\!1/p}$
  (or its measure-theoretic analogue). This provides a tunable continuum
  between averaging ($p$ small) and the infimum ($p\!\to\!\infty$), and is
  typically more robust to isolated outliers. A weighted variant
  $\sum_a w_a(1-\sem{F}_a)^p$ with $\sum_a w_a=1$ can encode station
  reliability.
\end{itemize}
In sensor networks with occasional false alarms (one station outputs~0 while
thousands report values near~1), product and infimum both collapse to~0. A
moderate-$p$ LTN aggregator (optionally weighted) avoids single-sensor
catastrophic failure while still converging to the strict infimum as
${p\to\infty}$. When a strict ``all must hold'' interpretation is intended and
measurements are trusted as independent, the product is appropriate.

\paragraph{Infinite (time) domains.}
For countably or uncountably infinite sets of (time) points $T$, instead of the finite set of TimeSlots from the previous example, we use the
measure-theoretic aggregations from Sec.~\ref{subsec:cont-prob-sem}
(Eqs.~\eqref{eq:prob-exists-quant}--\eqref{eq:prob-forall-quant}) or the
infinitary LTN aggregations from Sec.~\ref{subsec:inf-LTN}. Given a (probability)
measure $\rho_T$ on $T$ and writing $\phi(\tau):=\sem{F}_{\nu[\tau\mapsto \tau]}\in[0,1]$:
\[
  \sem{\forall \tau{:}T\,F}
    = \exp\Bigl(\,\expec_{\tau\sim\rho_T}[\ln \phi(\tau)]\Bigr),\qquad
  \sem{\exists \tau{:}T\,F}
    = 1-\exp\Bigl(\,\expec_{\tau\sim\rho_T}[\ln(1-\phi(\tau))]\Bigr).
\]
These reduce to certain (depending on the measure $\rho_T$) finite products/probabilistic sums when $T$ is finite with the
counting measure. 
Zeros at isolated times yield $\ln 0=-\infty$ and therefore a
value~$0$ for $\forall$ (consistent with the convention $0\cdot\infty=0$).

Alternatively, the infinitary LTN scheme provides a robust family of
aggregators:
\[
  \sem{\forall \tau{:}T\,F}
   = 1-\Bigl(\,\int_T (1-\phi(\tau))^{p}\,d\rho_T(\tau)\Bigr)^{\!1/p},\qquad
  \sem{\exists \tau{:}T\,F}
   = \Bigl(\,\int_T \phi(\tau)^{p}\,d\rho_T(\tau)\Bigr)^{\!1/p}.
\]
For rare glitches (bad weather reported at a few time instants), these
aggregators with moderate $p$ (or with $\rho_T$ down-weighting unlikely times)
avoid collapse to~0, while $p\to\infty$ recovers the strict infimum. When a
``almost everywhere'' reading is desired, one may also use the essential
infimum/supremum w.r.t.~$\rho_T$ to ignore measure-zero anomalies.

\subsection{Comparison to LDL (Logic of Differentiable Logics)}

The Logic of Differentiable Logics (LDL) approach \label{subsec:ldl}
\citep{slusarzLogicDifferentiableLogics2023} differs from our framework in several fundamental aspects. First, LDL incorporates comparison operators directly into their syntax, whereas our approach treats them as part of the non-logical signature, providing greater flexibility in language design. Second, LDL restricts their type system to Bool, Real, Vector, and Index types, while our framework maintains a more general type-theoretic foundation. Third, LDL's syntax includes constructs such as lambda-terms, let-terms, repetitions of expressions, vector constructors, and vector look-ups, which introduces complexity that our approach avoids through a more streamlined logical structure.

While both approaches employ typed languages, LDL does not utilise an abstract signature to achieve language generality. Furthermore, LDL lacks the algebraic structure provided by our 2Mon-BLat framework, which endows our semantics with a clearly defined mathematical structure while maintaining sufficient flexibility to accommodate different logical semantics (STL as an approximation). Additionally, LDL makes the restrictive assumption that all distinct random variables are independent, limiting its applicability to scenarios where this independence assumption holds.

A key distinction lies in the quantification mechanisms: LDL employs variable-dependent quantification, whereas our framework utilises sort-dependent quantification, which furthermore reduces complexity by providing all the (sort-dependent) information necessary for quantification by defining the interpretation function.

Finally, LDL does not require associativity of the logical operators, enabling the use of non-associative and/or operators characteristic of Signal Temporal Logic (STL). This non-associativity requirement constitutes another reason why STL does not naturally fit into our 2Mon-BLat model and should instead be modelled as approximating a 2Mon-BLat structure, a topic to be addressed in subsequent work. 

Actually, it is not absolutely necessary to require associativity, as we did in Def.~\ref{def:2Mon-BLat} of our 2Mon-BLat. However unifying nearly all semantics under this algebraic umbrella and seeing STL (and potentially others) as approximations of the 2Mon-BLat algebra, gives us a richer structure of our semantics and reveals otherwise suppressed algebraic laws.

\paragraph{LDL quantifiers} \label{par:ldl-quantifiers}
Recall Defs.~\ref{def:aggregated-2Mon-BLat} and~\ref{def:2Mon-BLat-cat} and this discussion following the latter: an aggregated $\mathbf{2Mon\text{-}BLat}$ provides maps $\mathrm{aggr}_X^\forall,\mathrm{aggr}_X^\exists:\mathcal L^X\to \mathcal L$, where additionally, a structure on $X$ may be used. In LDL, the context $Q$ maps a bound variable $x{:}s$ to a random variable $Q[x]$ on $\mathcal I(s)$ with density $p_X$. For a measurable $g: \mathcal I(s)\to \mathbb R$ write
\[
  x_{\min}:=\operatorname*{arg\,min}_{a\in\mathcal{I}(s)} g(a),\qquad
  x_{\max}:=\operatorname*{arg\,max}_{a\in\mathcal{I}(s)} g(a).
\]
For a random variable $X$ with density $p_X$, LDL defines in \citep[p.~9]{slusarzLogicDifferentiableLogics2023} (and we adopt) the quantifier aggregations exactly as
\[
  \mathrm E_{\min}[g(X)] := \lim_{\gamma\to 0} \int_{x\in B_\gamma^{x_{\min}}} p_X(x)\,g(x)\,dx,\qquad
  \mathrm E_{\max}[g(X)] := \lim_{\gamma\to 0} \int_{x\in B_\gamma^{x_{\max}}} p_X(x)\,g(x)\,dx.
\]
Hence we set
\[
  \mathrm{aggr}_{\mathcal I(s),Q[x]}^{\forall}(g) := \mathrm E_{\min}\bigl[g\bigl(Q[x]\bigr)\bigr],\qquad
  \mathrm{aggr}_{\mathcal I(s),Q[x]}^{\exists}(g) := \mathrm E_{\max}\bigl[g\bigl(Q[x]\bigr)\bigr].
\]
Consequently, for any $F$ and valuation $\nu$, this is coherent with the definition of the quantifiers from Table~\ref{tab:semantics}:
\[
  \sem{\forall x{:}s\,F}_{\nu} = \mathrm{aggr}_{\mathcal I(s),Q[x]}^{\forall}\bigl(\lambda a. \sem{F}_{\nu[x\mapsto a]}\bigr),\qquad
  \sem{\exists x{:}s\,F}_{\nu} = \mathrm{aggr}_{\mathcal I(s),Q[x]}^{\exists}\bigl(\lambda a. \sem{F}_{\nu[x\mapsto a]}\bigr).
\]

\paragraph{Variable–dependent vs. sort–dependent quantification.}
In our probabilistic semantics (Sec.~\ref{subsec:cont-prob-sem}), each sort $s$ comes with a fixed measure $\mu_s$ and quantification integrates with respect to $\mu_s$ (independent of the variable name). LDL instead equips each \emph{bound variable} $x{:}s$ with its own random variable $Q[x]$ (which has a density and may differ between variables of the same sort). This design choice serves two purposes:
\begin{itemize}
  \item It allows \emph{simultaneous} use of different distributions on the same sort in one formula, e.g. $x{:}s$ drawn from a data distribution $Q[x]$ and $y{:}s$ drawn from an adversarial or reweighted distribution $Q[y]$.
  \item It enables \emph{contextual} or \emph{conditional} sampling: $Q[x]$ can depend on the surrounding bound context $\Gamma$ or external parameters, effectively acting as a Markov kernel $\rho_x(\cdot\mid\Gamma)$.
\end{itemize}

\noindent
By contrast, our sort–based variant fixes a measure $\mu_s$ and uses the same aggregator for every variable $x{:}s$: for the chosen aggregation operators (product/infimum/LTN, etc.). This is much simpler, since it avoids the complexity of managing multiple distributions for different variables of the same sort. Moreover, it seems natural to associate the carrier sets of an interpretation with probability distributions. However, NeSyCat still fulfils all 3 goals stated in Sec. 1 of \cite{slusarzLogicDifferentiableLogics2023} and is even more modular and well-separated: 

\begin{enumerate}
  \item NeSyCat formally covers a sufficient fragment of first-order logic to express key properties in machine learning verification, such as robustness.
  \item The syntax, semantics and pragmatics of NeSyCat are well-separated. And even more so than in LDL, since we clearly disentangle the concepts of signature, syntax as context-free grammar, interpretation of a NeSy framework, and Tarskian semantics of a NeSy system in an 2Mon-BLat algebra.
  \item NeSyCat has a unified, general syntax and semantics able to express multiple different DLs and is modular on the choice of DL, by introducing the modularity of choosing the monad, the truth space, and the 2Mon-BLat algebra by choosing a NeSy framework.
\end{enumerate}

\paragraph{NeSyCat encoding of the LDL robustness example.}
\noindent
Following \cite[Ex.~3.1]{slusarzLogicDifferentiableLogics2023}, we consider an image-classification setting where inputs are $28\times 28$ grayscale images flattened to vectors in $\mathbb R^{784}$. Hence our input sort is $\mathrm{Vec}_{784}$ with $\mathcal I(\mathrm{Vec}_{784})=\mathbb R^{784}$. A classifier typically outputs a vector of class scores/logits in $\mathbb R^{m}$ (e.g., $m=10$ for MNIST). Accordingly, we take the network as a function
\[
  f: \mathbb R^{784} \longrightarrow \mathbb R^{m} \qquad (m\ge 1 \text{ fixed}).
\]
The robustness property we encode states $\ell_{\infty}$-robustness around a reference image $\hat x$: whenever an input $x$ lies in the $\ell_{\infty}$-ball $B_{\infty}(\hat x,\varepsilon)$, the output $f(x)$ must lie in the $\ell_{\infty}$-ball $B_{\infty}(f(\hat x),\delta)$. This is exactly what our predicate $\mathrm{bounded}(\cdot,\cdot,\cdot)$ and the formula $\Phi_{\varepsilon,\delta,\hat x}$ capture below.
Regarding the first goal mentioned above, we can encode the LDL robustness example as follows: Fix sorts $\mathrm{Vec}_{784},\,\mathrm{Index}_{784},\,\mathrm{Vec}_{m},\,\mathrm{Index}_{m},\,\mathrm{Real}$ with interpretations $\mathcal I(\mathrm{Vec}_{784})=\mathbb R^{784}$, $\mathcal I(\mathrm{Index}_{784})=\{0,\dots,783\}$, $\mathcal I(\mathrm{Vec}_{m})=\mathbb R^{m}$, $\mathcal I(\mathrm{Index}_{m})=\{0,\dots,m-1\}$, $\mathcal I(\mathrm{Real})=\mathbb R$. Let normal function symbols

\begin{align*}
  & \mathrm{at}_{784}:\mathrm{Vec}_{784}\times\mathrm{Index}_{784}\to\mathrm{Real},\quad
  && \mathrm{at}_{m}:\mathrm{Vec}_{m}\times\mathrm{Index}_{m}\to\mathrm{Real},\quad \\
  & \mathrm{abs}:\mathrm{Real}\to\mathrm{Real},\quad
  && \mathrm{leq}:\mathrm{Real}\times\mathrm{Real}\to\Omega,
\end{align*}
be given, where $\mathrm{at}_{784}(v,i)$ and $\mathrm{at}_{m}(w,j)$ read components, $\mathrm{abs}$ is absolute value, and $\mathrm{leq}(a,b)$ is the crisp predicate $[a\le b]$. Let $f{:}\mathrm{Vec}_{784}\to\mathrm{Vec}_{m}$ be a computational function symbol (the network). Define the derived predicates
\[
  \mathrm{bounded\_in}(v,u,a)
   := \forall i{:}\mathrm{Index}_{784}\;\mathrm{leq}\bigl(\mathrm{abs}(\mathrm{at}_{784}(v,i)-\mathrm{at}_{784}(u,i)),\,a\bigr),
\]
\[
  \mathrm{bounded\_out}(w,z,a)
   := \forall j{:}\mathrm{Index}_{m}\;\mathrm{leq}\bigl(\mathrm{abs}(\mathrm{at}_{m}(w,j)-\mathrm{at}_{m}(z,j)),\,a\bigr).
\]
For parameters $\varepsilon,\delta\in\mathbb R$ and a fixed input $\hat x\in\mathbb R^{784}$, the LDL robustness property of $f$ is encoded in NeSyCat as the family of formulas
\[
  \Phi_{\varepsilon,\delta,\hat x}
   := \forall x{:}\mathrm{Vec}_{784}\;\Bigl(\mathrm{bounded\_in}(x,\hat x,\varepsilon) \;\rightarrow\; \mathrm{bounded\_out}\bigl(f(x),f(\hat x),\delta\bigr)\Bigr).
\]
Under the LDL quantifier aggregator (par.~\ref{par:ldl-quantifiers}) for the bound variable $x$, its semantics is
\[
\begin{aligned}
  \sem{\Phi_{\varepsilon,\delta,\hat x}}_{\nu}
   &= \mathrm{aggr}^{\forall}_{\mathcal I(\mathrm{Vec}_{784}),\,Q[x]}
         \bigl( \lambda a. \sem{\mathrm{bounded\_in}(x,\hat x,\varepsilon) \rightarrow \mathrm{bounded\_out}(f(x),f(\hat x),\delta)}_{\nu[x\mapsto a]} \bigr) \\
   &= \mathrm E_{\min}\bigl[g(Q[x])\bigr]
\end{aligned}
\]
where the semantics of $g(a):=\lambda a. \sem{\mathrm{bounded\_in}(x,\hat x,\varepsilon) \rightarrow \mathrm{bounded\_out}(f(x),f(\hat x),\delta)}_{\nu[x\mapsto a]}$ is given by the semantics of the implication operator in a certain 2Mon-BLat. In conclusion we have modelled our example to coincide with the LDL semantics.

\subsection{Monadic Semantics and Probabilistic Programming}\label{subsec:monadic-pp}

Our use of monads for structuring neurosymbolic semantics draws on a rich tradition of monadic approaches to probabilistic programming. \cite{moggiNotionsComputationMonads1991} established monads as the uniform abstraction for computational effects; the application to probability was pioneered by \cite{giry1982categorical}, who introduced the probability monad on measurable spaces, and by \cite{jones1989probabilistic}, who developed the analogous probabilistic powerdomain in domain theory. \cite{ramsey2002stochastic} made the connection to programming languages explicit via a stochastic lambda calculus in which probability distributions form a monad with unit given by the Dirac delta and bind given by marginalisation.

On the semantic side, \cite{staton2016semantics} gave the first adequate measure-theoretic denotational semantics for a probabilistic programming metalanguage supporting higher-order functions, continuous distributions, and soft constraints (conditioning), using Kleisli morphisms of a measure monad as the central semantic objects---the same kind of compositional, Kleisli-based semantics that NeSyCat employs. \cite{staton2017commutative} further showed that the semantics based on s-finite kernels is commutative, meaning that independent computations can be reordered freely; this property is directly relevant to NeSyCat, where independent parts of a neurosymbolic system may be evaluated in any order. A tutorial exposition of the ``programs as measures'' viewpoint is given in \citep{staton2020programs}.

On the categorical side, \cite{fritz_synthetic_2020} developed the theory of Markov categories, showing that the Kleisli category of any affine commutative monad (including the Giry monad) is a Markov category. In particular, $\Kl{\mathcal{G}} \cong \Stoch$ is a Markov category, and fundamental results of probability theory---conditional independence, disintegration, sufficient statistics---can be stated and proved at this abstract level, applying uniformly to all NeSyCat instantiations. \cite{fritz_perrone2018bimonoidal} complements this with a bimonoidal treatment of joints, marginals, and independence for probability monads. \cite{cho_disintegration_2019} developed disintegration and Bayesian inversion categorically via string diagrams, and \cite{jacobs2018probability} systematically connected probability monads to commutative effectuses---a categorical framework for probabilistic reasoning supporting predicates, conditioning, and Bayesian inversion.

A central limitation of the category $\Meas$ is that it is not Cartesian closed, preventing a straightforward higher-order probabilistic semantics. \cite{heunenConvenientCategoryHigherOrder2017} resolved this by introducing quasi-Borel spaces ($\mathbf{QBS}$), which carry a probability monad $P$ that, when restricted to standard Borel spaces, agrees with the Giry monad. Since NeSyCat currently works within standard Borel spaces (or measure spaces) for its probabilistic frameworks, our semantics embeds conservatively into the $\mathbf{QBS}$ setting, which would become necessary for higher-order extensions (e.g., neural networks as first-class probabilistic objects). \cite{vakar2019domain} extended this to quasi-Borel predomains supporting recursion and conditioning, and \cite{scibior2018denotational} used the $\mathbf{QBS}$ semantics to formally validate inference algorithms in the monad-bayes library \citep{scibior2015}.

In summary, NeSyCat sits at a previously unexplored intersection: it applies the same categorical and monadic foundations developed for probabilistic programming (Giry monad, Kleisli categories, Markov categories) to the problem of unifying neurosymbolic systems, rather than to probabilistic languages per se. This connection suggests that results from categorical probability---such as Fritz's abstract conditional independence or Cho--Jacobs disintegration---may be directly applicable to structured reasoning about neurosymbolic systems.

\subsection{Categorical Logic and Categories for AI}
\label{sec:cat-logic}

The use of category-theoretic structure in logic long predates current neurosymbolic and machine-learning applications. \cite{makkaiReyesFirstOrderCategorical1977} provide a classical account of first-order categorical logic, and Johnstone's later exposition \cite{johnstoneSketchesElephantTopos2002} develops the same tradition in a form that has informed our NeSyCat semantics. NeSyCat does not aim to replace this line of work; rather, it keeps the surface syntax of first-order logic while extending the semantic side through monads and algebraically parametrised truth-value aggregation through the 2Mon-BLat.

Within categories for AI, categorical NLP provides another important point of contact. In standard DisCoCat, grammatical types are mapped to vector spaces and grammatical reductions to linear maps, yielding a compositional model of meaning based on string diagrams \citep{coeckeDisCoCatMathematicalFoundations2010}. Especially relevant for our purposes is higher-order DisCoCat \citep{toumiHigherOrderDisCoCatPeirceLambekMontague2025}, which reconnects categorical compositional semantics with Montague-style first-order logic by treating word meanings as diagram-valued higher-order functions and recovering negation and quantification in a string-diagrammatic setting. This is close in spirit to NeSyCat because it shows how categorical structure can recover genuinely logical operations, although its primary focus is natural-language semantics rather than monadic semantics for neurosymbolic systems.

A more recent line applies category theory directly to AI and machine learning. \cite{fongBackpropFunctor2019} show that supervised learning and backpropagation admit a compositional functorial formulation, and Spivak's broader applied-category-theory perspective \citep{Spivak2014} helped make such compositional methods accessible well beyond pure category theory. This is close in spirit to NeSyCat, since both emphasise compositional structure and parametrisation, but the focus is different: their work formalises learning dynamics, whereas our concern here is the interpretation of logical formulas in monadic and algebraic neurosymbolic semantics.

Likewise, \cite{gavranovicPositionCategoricalDeep2024} argue that categorical deep learning offers a unifying algebraic account of neural network architectures and parametric maps. This complements rather than competes with NeSyCat: they focus on architecture and learnable composition, whereas we focus on signatures, interpretations, and semantic evaluation. In this sense, NeSyCat can be read as a bridge between categorical logic and categories for AI, importing categorical and monadic structure into neurosymbolic semantics rather than only into architecture design.

\section{Implementation of NeSy Frameworks}
\label{sec:implementation}

We have implemented NeSy frameworks in 2 \texttt{NeSyCat} libraries, which are available at \url{https://github.com/NeSyCat/PyTorch} for the PyTorch version and at \url{https://github.com/NeSyCat/HaskTorch} for the HaskTorch version. The library comes with predefined  NeSy frameworks, but also allows users to define their own frameworks, by providing a monad, a truth value space and a double monoid bounded lattice. The library also supports the definition of interpretations. A parser transforms formulas into an abstract syntax tree. There is a function implementing Tarskian semantics, i.e.\ the evaluation of formulas in a given interpretation. The library also includes a module for NeSy transformations, allowing users to apply transformations between different NeSy frameworks.
For simplicity, we have not implemented a sort system, which means that the implementation is untyped.\footnote{In Haskell, we could use type families and heterogeneous lists to implement sorted interpretations.}
Also, the integration of neural networks into interpretations has not been implemented yet, but we plan to do so in the future.

The rich Haskell type system allows us to express the semantics of NeSy frameworks in a type-safe way. The Python implementation takes this as a role model, but less type-safe, because Python is dynamically typed. We start with describing the Haskell implementation. We first introduce a type class for double monoid bounded lattices. Note that for simplicity, we have not implemented the lattice structure, because it is not used in the semantics.\footnote{In the future, this needs to be added. LDL quantifiers use the lattice structure. Also, to use ULLER for neuro-symbolic learning and reasoning, we need the lattice structure, because we need to be able maximise and minimise over the truth values as described in \citep{vankriekenULLERUnifiedLanguage2024}.}

\begin{minted}{haskell}
class TwoMonBLat a where
  top, bot :: a
  neg :: a -> a
  conj, disj, implies :: a -> a -> a
\end{minted}

\noindent
Note that \minth{conj} and \minth{disj} are typed as plain binary operations \minth{a -> a -> a}, without an explicit \minth{Monoid} constraint. This is intentional: the monoid laws (associativity and neutral elements, cf.\ Def.~\ref{def:2Mon-BLat} and Remark~\ref{rem:uller-comparison}) are mathematical requirements on the algebra not enforced by Haskell's type system. The implementation trusts that concrete instances satisfy these laws by construction -- the predefined algebras in our library do so, and users designing custom algebras should verify the laws mathematically. Enforcing algebraic laws at the type level (e.g., via property-based testing or dependent types) is left for future work.

Based on this, we define a type class for aggregated double monoid bounded lattices:
\begin{minted}{haskell}
class TwoMonBLat a => Aggr2MonBLat s a where
  -- for a structure on b and a predicate on b, aggregate truth values a 
  aggrE, aggrA :: s b -> (b -> a) -> a
\end{minted}

Note that aggregation takes into account the structure of our category $\category{C}$ of sets with structure, represented as \texttt{s b} here. Usually, we use finite lists, and then aggregation is just iteration of disjunction or conjunction:
\begin{minted}{haskell}
-- the mainly used Aggr2MonBLat: no additional structure (just lists) + Booleans
instance Monad t => Aggr2MonBLat [] (t Bool) where
  aggrE s f = foldr disj bot $ map f s
  aggrA s f = foldr conj top $ map f s
\end{minted}

For infinite aggregation, we need \texttt{s b} to be some (probability) measure. For example, for implementing the Giry monad, we can use the \texttt{Integrator} monad \citep{tobin2018embedded} from the \texttt{monad-bayes} package, which represents measure spaces in a very faithful way. the Giry monad can be defined as a submonad. Values of this monad can be constructed from values of the \texttt{Integrator} monad by normalising them:

\begin{minted}{haskell}
newtype Giry a = Giry { runGiry :: Integrator a }
fromIntegrator :: Integrator a -> Giry a
fromIntegrator m = Giry $ normalize $ lift m
\end{minted}

Then, aggregation can be defined as in equation~\ref{eq:prob-exists-quant} and \ref{eq:prob-forall-quant}. Note that we need a double integral (\texttt{runIntegrator}) here, because the integrand is itself an element of ${\mathcal T} Bool$, and we need to employ the isomorphism $Giry\ Bool \cong [0,1]$ for the \texttt{Giry} monad. 

\begin{minted}{haskell}
instance Aggr2MonBLat Integrator (Giry Bool) where
  aggrA meas f = 
    Giry $ integrator $ \meas_fun ->
        exp $ runIntegrator (runIntegrator (log . meas_fun) . runGiry . f) meas
  aggrE meas f =
    neg (aggrA meas (neg . f))
\end{minted}

However, the \texttt{Integrator} monad is not very efficient and does not provide sampling. Therefore, we also provide an instance for the \texttt{SamplerIO} monad from the same package, which implements the Giry monad using sampling. Here, using Monte Carlo integration, we can compute the infinite aggregation up to any given precision by increasing the number of samples.  Note that first a value in the domain is sampled, then the computational predicate is applied to the value, resulting in an element of ${\mathcal T} Bool$, and finally a Boolean value is sampled from that. In the case of universal quantification, we aggregate the sampled Boolean values using logical conjunction. The probability of obtaining \texttt{True} in all cases is the product of the probabilities of obtaining \texttt{True} for each case. This finite product approximates the infinite product (expressed using $\exp$ and $\ln$) in equation~\ref{eq:prob-exists-quant}. A similar reasoning holds for existential quantification.

\begin{minted}{haskell}
-- Expectation-style aggregation over a distribution
-- Here we approximate via Monte Carlo with no_samples samples
no_samples = 1000
aggregation :: Monad m => ([a] -> a) -> m b -> (b -> m a) -> m a
aggregation connective dist f = do
  samples <- sequence (replicate no_samples dist)
  vals    <- mapM f samples
  return (connective vals)
instance Aggr2MonBLat SamplerIO (SamplerIO Bool) where
  aggrE = aggregation or
  aggrA = aggregation and
\end{minted}
 
Based on this type class and the predefined type constructor class for monads, we define a type class for NeSy frameworks, as well as various instances. The instances do not to define any methods, because these  have already been defined in the superclasses. 
\begin{minted}{haskell}
class (Monad t, Aggr2MonBLat s (t omega)) => NeSyFramework t s omega 
-- Classical instance using identity monad, Omega is Bool
instance NeSyFramework Identity [] Bool 
-- Distribution instance, Omega is Bool
instance Num prob => NeSyFramework (Dist.T prob) [] Bool 
-- Non-empty powerset instance (non-determinism)
instance NeSyFramework SM.Set [] Bool 
-- Giry monad instance, using SamplerIO for both aggregation and the monad
instance NeSyFramework SamplerIO SamplerIO Bool
-- Giry monad instance, using Integrator for aggregation and Giry for the monad
instance NeSyFramework Giry Integrator Bool
\end{minted}

Next, we show the type definition for NeSy interpretations, which is parameterised by the monad \texttt{t}, the structure \texttt{s} of the category $\category{C}$, the truth value space \texttt{omega} and a type \texttt{a} for the universe of discourse. 

\begin{minted}{haskell}
  data Interpretation t s omega a =
  Interpretation { universe :: s a,
                   funcs :: Map.Map Ident ([a] -> a),
                   mFuncs :: Map.Map Ident ([a] -> t a),
                   preds :: Map.Map Ident ([a] -> omega),
                   mPreds :: Map.Map Ident ([a] -> t omega) }
\end{minted}

\noindent
We refrain from showing the same number of details of the Python implementation here, because it is much more verbose than the Haskell implementation.
The Python implementation follows a structure similar to that of the Haskell implementation, but uses Python's dynamic typing and built-in data structures, building on the \mintp{pymonad} package. Here is the Python code snippet for the NeSy framework type class:
\begin{minted}{python}
class NeSyFramework[_T: ParametrizedMonad, _O, _R: Aggr2MonBLat]:
    """
    Class to represent a monadic NeSy framework consisting of a monad (T),
    a set Omega acting as truth basis (O),
    and an aggregated double monoid bounded lattice (R).
    This class ensures the following runtime constraint which is not
    representable in Pythons type system:
    - _R: Aggr2MonBLat[_T[_O]]
    """
    _monad: Type[_T]
    _logic: _R
    ...
\end{minted}

As in the case of Haskell, also for Python, we provide two versions of the Giry monad, one based on integration and one based on sampling. Again, the integration version is more faithful to the mathematical definition of the Giry monad, but the sampling version is more efficient. The sampling version is based on the \texttt{numpy} library. An idea for a more abstract version would be to use the \texttt{PyMC} library, such that Bayesian inference becomes possible. However, a principal problem arises. For \texttt{PyMC}, only an internal monad can be defined, providing monadic lifting for functions depending on \texttt{TensorVariable}s. However, in order to obtain a monad in Python, we would need to lift functions depending on normal Python variables. This is a topic for future work. Once the code base of the original ULLER paper \citep{vankriekenULLERUnifiedLanguage2024} is available, our code base could be used to enhance the ULLER implementation with our modular abstractions.

\section{Conclusion}
\label{sec:conclusion}

The ULLER language \cite{vankriekenULLERUnifiedLanguage2024} aims at a unifying foundation for neurosymbolic systems. In this paper, we have developed a new semantics for ULLER, based on Moggi's formalisation of computational effects as monads. In contrast to the original semantics, our semantics is truly modular. It is based on a notion of NeSy framework that provides the structure of the computational effects and the space of truth values. This modularity will enable a cleaner, more modular implementation of ULLER in Python and an easier integration of new frameworks, as well as a structured method of translating between different frameworks. First implementations of our NeSyCat framework are available in Python and Haskell, see \url{https://github.com/cherryfunk/NeSyCat}.
Note that the distributional and probabilistic NeSy frameworks will make parameterized interpretations in the sense of \cite{vankriekenULLERUnifiedLanguage2024} differentiable and that this can be integrated by using a category of differentiable manifolds and functions. However, a detailed investigation of differentiability, computability, and computational complexity of our approach is left for future work in a subsequent paper dedicated to NeSyCat's implementation.

Our work suggests an analogy between ULLER's formulas $[x:=m(T_1,\dots,T_n)]F$ and Haskell's do-notation $\textbf{do } x\leftarrow m(T_1,\dots,T_n); F$ for computational effects. Inspired by this analogy, one could extend ULLER to a language with computational terms and formulas that may be nested.  

\section*{Acknowledgments}
We thank Rick Adamy for the idea to use monads in the context of ULLER, Kai-Uwe K\"uhnberger for initiating our collaboration, Emile van Krieken for useful discussions and Bj\"orn Gehrke for helping us with the implementation in Python.
From Daniel: Big thanks to my father and family Bircks for providing me the space-time to work on this paper, Alice for reminding me to eat, and to my mother and Leilani H. Gilpin for mental support.

\bibliographystyle{abbrvnat}
\bibliography{daniel_zotero,additional_citations}

\begin{thebibliography}{44}
\providecommand{\natexlab}[1]{#1}
\providecommand{\url}[1]{\texttt{#1}}
\expandafter\ifx\csname urlstyle\endcsname\relax
  \providecommand{\doi}[1]{doi: #1}\else
  \providecommand{\doi}{doi: \begingroup \urlstyle{rm}\Url}\fi

\bibitem[Ad\'{a}mek et~al.(1990)Ad\'{a}mek, Herrlich, and Strecker]{AHS}
J.~Ad\'{a}mek, H.~Herrlich, and G.~Strecker.
\newblock \emph{Abstract and Concrete Categories}.
\newblock Wiley, New York, 1990.

\bibitem[Awodey(2010)]{awodeyCategoryTheory2010}
S.~Awodey.
\newblock \emph{Category Theory}.
\newblock Number~52 in Oxford Logic Guides. Oxford Univ. Press, Oxford, 2. ed
  edition, 2010.
\newblock ISBN 978-0-19-958736-0 978-0-19-923718-0.

\bibitem[Badreddine and Spranger(2021)]{DBLP:conf/nesy/BadreddineS21}
S.~Badreddine and M.~Spranger.
\newblock Extending real logic with aggregate functions.
\newblock In A.~S. d'Avila Garcez and E.~Jim{\'{e}}nez{-}Ruiz, editors,
  \emph{Proceedings of the 15th International Workshop on Neural-Symbolic
  Learning and Reasoning as part of the 1st International Joint Conference on
  Learning {\&} Reasoning {(IJCLR} 2021), Virtual conference, October 25-27,
  2021}, volume 2986 of \emph{{CEUR} Workshop Proceedings}, pages 115--125.
  CEUR-WS.org, 2021.
\newblock URL \url{https://ceur-ws.org/Vol-2986/paper9.pdf}.

\bibitem[Badreddine et~al.(2022)Badreddine, d'Avila Garcez, Serafini, and
  Spranger]{DBLP:journals/ai/BadreddineGSS22}
S.~Badreddine, A.~S. d'Avila Garcez, L.~Serafini, and M.~Spranger.
\newblock Logic tensor networks.
\newblock \emph{Artif. Intell.}, 303:\penalty0 103649, 2022.
\newblock \doi{10.1016/J.ARTINT.2021.103649}.
\newblock URL \url{https://doi.org/10.1016/j.artint.2021.103649}.

\bibitem[Cho and Jacobs(2019)]{cho_disintegration_2019}
K.~Cho and B.~Jacobs.
\newblock Disintegration and {Bayesian} inversion via string diagrams.
\newblock \emph{Mathematical Structures in Computer Science}, 29\penalty0
  (7):\penalty0 938--971, 2019.
\newblock \doi{10.1017/S0960129518000488}.

\bibitem[Coecke et~al.(2010)Coecke, Sadrzadeh, and
  Clark]{coeckeDisCoCatMathematicalFoundations2010}
B.~Coecke, M.~Sadrzadeh, and S.~Clark.
\newblock ({{DisCoCat}}) {{Mathematical Foundations}} for a {{Compositional
  Distributional Model}} of {{Meaning}}, Mar. 2010.

\bibitem[Fong et~al.(2019)Fong, Spivak, and
  Tuy{\'e}ras]{fongBackpropFunctor2019}
B.~Fong, D.~I. Spivak, and R.~Tuy{\'e}ras.
\newblock Backprop as functor: A compositional perspective on supervised
  learning.
\newblock In \emph{2019 34th Annual ACM/IEEE Symposium on Logic in Computer
  Science (LICS)}, pages 1--13. IEEE, 2019.
\newblock \doi{10.1109/LICS.2019.8785665}.

\bibitem[Fritz(2020)]{fritz_synthetic_2020}
T.~Fritz.
\newblock A synthetic approach to {Markov} kernels, conditional independence
  and theorems on sufficient statistics.
\newblock \emph{Advances in Mathematics}, 370:\penalty0 107239, 2020.
\newblock \doi{10.1016/j.aim.2020.107239}.

\bibitem[Fritz and Perrone(2018)]{fritz_perrone2018bimonoidal}
T.~Fritz and P.~Perrone.
\newblock Bimonoidal structure of probability monads.
\newblock \emph{Electronic Notes in Theoretical Computer Science},
  341:\penalty0 121--149, 2018.
\newblock \doi{10.1016/j.entcs.2018.11.007}.

\bibitem[Gavranovi{\'c} et~al.(2024)Gavranovi{\'c}, Lessard, Dudzik, von Glehn,
  Ara{\'u}jo, and Veli{\v c}kovi{\'c}]{gavranovicPositionCategoricalDeep2024}
B.~Gavranovi{\'c}, P.~Lessard, A.~Dudzik, T.~von Glehn, J.~G.~M. Ara{\'u}jo,
  and P.~Veli{\v c}kovi{\'c}.
\newblock Position: {{Categorical Deep Learning}} is an {{Algebraic Theory}} of
  {{All Architectures}}, June 2024.

\bibitem[Girard(1995)]{girardLinearLogicIts1995}
J.-Y. Girard.
\newblock Linear {{Logic}}: Its syntax and semantics.
\newblock In J.-Y. Girard, Y.~Lafont, and L.~Regnier, editors, \emph{Advances
  in {{Linear Logic}}}, pages 1--42. Cambridge University Press, 1 edition,
  June 1995.
\newblock ISBN 978-0-521-55961-4 978-0-511-62915-0.
\newblock \doi{10.1017/CBO9780511629150.002}.

\bibitem[Giry(1982)]{giry1982categorical}
M.~Giry.
\newblock A categorical approach to probability theory.
\newblock In \emph{Categorical Aspects of Topology and Analysis}, volume 915 of
  \emph{Lecture Notes in Mathematics}, pages 68--85, Berlin, Heidelberg, 1982.
  Springer.

\bibitem[H{\'a}jek(1998)]{hajekMetamathematicsFuzzyLogic1998}
P.~H{\'a}jek.
\newblock \emph{Metamathematics of {{Fuzzy Logic}}}.
\newblock Trends in {{Logic}}. Springer Netherlands, Dordrecht, 1998.
\newblock ISBN 978-1-4020-0370-7 978-94-011-5300-3.
\newblock \doi{10.1007/978-94-011-5300-3}.

\bibitem[Harel et~al.(2001)Harel, Kozen, and Tiuryn]{harel2001dynamic}
D.~Harel, D.~Kozen, and J.~Tiuryn.
\newblock Dynamic logic.
\newblock \emph{ACM SIGACT News}, 32\penalty0 (1):\penalty0 66--69, 2001.

\bibitem[Heunen et~al.(2017)Heunen, Kammar, Staton, and
  Yang]{heunenConvenientCategoryHigherOrder2017}
C.~Heunen, O.~Kammar, S.~Staton, and H.~Yang.
\newblock A {{Convenient Category}} for {{Higher-Order Probability Theory}}.
\newblock In \emph{2017 32nd {{Annual ACM}}/{{IEEE Symposium}} on {{Logic}} in
  {{Computer Science}} ({{LICS}})}, pages 1--12, June 2017.
\newblock \doi{10.1109/LICS.2017.8005137}.

\bibitem[Hájek(1998)]{hajek_metamathematics_1998}
P.~Hájek.
\newblock \emph{Metamathematics of {Fuzzy} {Logic}}.
\newblock Trends in {Logic}. Springer Netherlands, Dordrecht, 1998.
\newblock ISBN 978-1-4020-0370-7 978-94-011-5300-3.
\newblock \doi{10.1007/978-94-011-5300-3}.
\newblock URL \url{http://link.springer.com/10.1007/978-94-011-5300-3}.

\bibitem[Jacobs(2018)]{jacobs2018probability}
B.~Jacobs.
\newblock From probability monads to commutative effectuses.
\newblock \emph{Journal of Logical and Algebraic Methods in Programming},
  94:\penalty0 200--237, 2018.
\newblock \doi{10.1016/j.jlamp.2016.11.006}.

\bibitem[Johnstone(2002)]{johnstoneSketchesElephantTopos2002}
P.~T. Johnstone.
\newblock \emph{Sketches of an {{Elephant A Topos Theory Compendium}}}.
\newblock Oxford University PressOxford, Sept. 2002.
\newblock ISBN 978-0-19-851598-2 978-1-383-02287-2.
\newblock \doi{10.1093/oso/9780198515982.001.0001}.

\bibitem[Jones and Plotkin(1989)]{jones1989probabilistic}
C.~Jones and G.~Plotkin.
\newblock A probabilistic powerdomain of evaluations.
\newblock In \emph{Proceedings of the Fourth Annual Symposium on Logic in
  Computer Science (LICS)}, pages 186--195, 1989.
\newblock \doi{10.1109/LICS.1989.39173}.

\bibitem[Mac~Lane(1978)]{maclaneCategoriesWorkingMathematician1978}
S.~Mac~Lane.
\newblock \emph{Categories for the {{Working Mathematician}}}, volume~5 of
  \emph{Graduate {{Texts}} in {{Mathematics}}}.
\newblock Springer New York, New York, NY, 1978.
\newblock ISBN 978-1-4419-3123-8 978-1-4757-4721-8.
\newblock \doi{10.1007/978-1-4757-4721-8}.

\bibitem[Makkai and Reyes(1977)]{makkaiReyesFirstOrderCategorical1977}
M.~Makkai and G.~E. Reyes.
\newblock \emph{First Order Categorical Logic: Model-Theoretical Methods in the
  Theory of Topoi and Related Categories}, volume 611 of \emph{Lecture Notes in
  Mathematics}.
\newblock Springer, Berlin, Heidelberg, 1977.
\newblock \doi{10.1007/BFb0066201}.

\bibitem[Manhaeve et~al.(2021)Manhaeve, Dumancic, Kimmig, Demeester, and
  Raedt]{DBLP:journals/ai/ManhaeveDKDR21}
R.~Manhaeve, S.~Dumancic, A.~Kimmig, T.~Demeester, and L.~D. Raedt.
\newblock Neural probabilistic logic programming in {DeepProbLog}.
\newblock \emph{Artif. Intell.}, 298:\penalty0 103504, 2021.
\newblock \doi{10.1016/J.ARTINT.2021.103504}.
\newblock URL \url{https://doi.org/10.1016/j.artint.2021.103504}.

\bibitem[Moggi(1991)]{moggiNotionsComputationMonads1991}
E.~Moggi.
\newblock Notions of computation and monads.
\newblock \emph{Information and Computation}, 93\penalty0 (1):\penalty0 55--92,
  July 1991.
\newblock ISSN 08905401.
\newblock \doi{10.1016/0890-5401(91)90052-4}.

\bibitem[Morettin et~al.(2017)Morettin, Passerini, and
  Sebastiani]{morettinEfficientWeightedModel2017}
P.~Morettin, A.~Passerini, and R.~Sebastiani.
\newblock Efficient {{Weighted Model Integration}} via {{SMT-Based Predicate
  Abstraction}}.
\newblock In \emph{Proceedings of the {{Twenty-Sixth International Joint
  Conference}} on {{Artificial Intelligence}}}, pages 720--728, Melbourne,
  Australia, Aug. 2017. International Joint Conferences on Artificial
  Intelligence Organization.
\newblock \doi{10.24963/ijcai.2017/100}.

\bibitem[Mossakowski et~al.(2010)Mossakowski, Schr{\"{o}}der, and
  Goncharov]{DBLP:journals/fac/MossakowskiSG10}
T.~Mossakowski, L.~Schr{\"{o}}der, and S.~Goncharov.
\newblock A generic complete dynamic logic for reasoning about purity and
  effects.
\newblock \emph{Formal Aspects Comput.}, 22\penalty0 (3-4):\penalty0 363--384,
  2010.
\newblock \doi{10.1007/S00165-010-0153-4}.
\newblock URL \url{https://doi.org/10.1007/s00165-010-0153-4}.

\bibitem[Priest(2008)]{priest2008introduction}
G.~Priest.
\newblock \emph{An introduction to non-classical logic: From if to is}.
\newblock Cambridge University Press, 2008.

\bibitem[Ramsey and Pfeffer(2002)]{ramsey2002stochastic}
N.~Ramsey and A.~Pfeffer.
\newblock Stochastic lambda calculus and monads of probability distributions.
\newblock In \emph{Proceedings of the 29th ACM SIGPLAN-SIGACT Symposium on
  Principles of Programming Languages (POPL)}, pages 154--165, 2002.
\newblock \doi{10.1145/503272.503288}.

\bibitem[Schellhorn and Mossakowski(2025)]{pmlr-v284-schellhorn25a}
D.~R. Schellhorn and T.~Mossakowski.
\newblock muller: A modular monad-based semantics of the neurosymbolic uller
  framework.
\newblock In L.~H.~Gilpin, E.~Giunchiglia, P.~Hitzler, and E.~van Krieken,
  editors, \emph{Proceedings of The 19th International Conference on
  Neurosymbolic Learning and Reasoning}, volume 284 of \emph{Proceedings of
  Machine Learning Research}, pages 494--518. PMLR, 08--10 Sep 2025.
\newblock URL \url{https://proceedings.mlr.press/v284/schellhorn25a.html}.

\bibitem[{\'S}cibior et~al.(2015){\'S}cibior, Ghahramani, and
  Gordon]{scibior2015}
A.~{\'S}cibior, Z.~Ghahramani, and A.~D. Gordon.
\newblock Practical probabilistic programming with monads.
\newblock In \emph{Proceedings of the 2015 ACM SIGPLAN Symposium on Haskell},
  pages 165--176, 2015.

\bibitem[{{\'S}cibior} et~al.(2018){{\'S}cibior}, Kammar, Vakar, Staton, Yang,
  Cai, Ostermann, Moss, Heunen, and Ghahramani]{scibior2018denotational}
A.~{{\'S}cibior}, O.~Kammar, M.~Vakar, S.~Staton, H.~Yang, Y.~Cai,
  K.~Ostermann, S.~K. Moss, C.~Heunen, and Z.~Ghahramani.
\newblock Denotational validation of higher-order {Bayesian} inference.
\newblock volume~2, 2018.
\newblock \doi{10.1145/3158148}.

\bibitem[{\'S}lusarz et~al.(2023){\'S}lusarz, Komendantskaya, Daggitt, Stewart,
  and Stark]{slusarzLogicDifferentiableLogics2023}
N.~{\'S}lusarz, E.~Komendantskaya, M.~L. Daggitt, R.~Stewart, and K.~Stark.
\newblock Logic of {{Differentiable Logics}}: {{Towards}} a {{Uniform
  Semantics}} of {{DL}}, Oct. 2023.

\bibitem[Smet et~al.(2023)Smet, Martires, Manhaeve, Marra, Kimmig, and
  Raedt]{smetNeuralProbabilisticLogic2023}
L.~D. Smet, P.~Z.~D. Martires, R.~Manhaeve, G.~Marra, A.~Kimmig, and L.~D.
  Raedt.
\newblock Neural {{Probabilistic Logic Programming}} in {{Discrete-Continuous
  Domains}}, Mar. 2023.

\bibitem[Spivak(2014)]{Spivak2014}
D.~I. Spivak.
\newblock \emph{Category Theory for the Sciences}.
\newblock The MIT Press, Cambridge, Massachusetts, 2014.
\newblock ISBN 978-0-262-02813-4.

\bibitem[Staton(2017)]{staton2017commutative}
S.~Staton.
\newblock Commutative semantics for probabilistic programming.
\newblock In \emph{Programming Languages and Systems (ESOP 2017)}, volume 10201
  of \emph{LNCS}, pages 855--879. Springer, 2017.
\newblock \doi{10.1007/978-3-662-54434-1_32}.

\bibitem[Staton(2020)]{staton2020programs}
S.~Staton.
\newblock Probabilistic programs as measures.
\newblock In G.~Barthe, J.-P. Katoen, and A.~Silva, editors, \emph{Foundations
  of Probabilistic Programming}, pages 43--74. Cambridge University Press,
  2020.
\newblock \doi{10.1017/9781108770750.003}.

\bibitem[Staton et~al.(2016)Staton, Yang, Wood, Heunen, and
  Kammar]{staton2016semantics}
S.~Staton, H.~Yang, F.~Wood, C.~Heunen, and O.~Kammar.
\newblock Semantics for probabilistic programming: higher-order functions,
  continuous distributions, and soft constraints.
\newblock In \emph{Proceedings of the 31st Annual ACM/IEEE Symposium on Logic
  in Computer Science (LICS)}, pages 525--534, 2016.
\newblock \doi{10.1145/2933575.2935313}.

\bibitem[Tobin(2018)]{tobin2018embedded}
J.~Tobin.
\newblock \emph{Embedded Domain-Specific Languages for Bayesian Modelling and
  Inference}.
\newblock PhD thesis, University of Auckland, 2018.

\bibitem[Toumi and
  de~Felice(2025)]{toumiHigherOrderDisCoCatPeirceLambekMontague2025}
A.~Toumi and G.~de~Felice.
\newblock Higher-{{Order DisCoCat}} ({{Peirce-Lambek-Montague}} semantics).
\newblock \emph{Electronic Proceedings in Theoretical Computer Science},
  429:\penalty0 130--145, Sept. 2025.
\newblock ISSN 2075-2180.
\newblock \doi{10.4204/EPTCS.429.7}.

\bibitem[Vakar et~al.(2019)Vakar, Kammar, and Staton]{vakar2019domain}
M.~Vakar, O.~Kammar, and S.~Staton.
\newblock A domain theory for statistical probabilistic programming.
\newblock \emph{Proceedings of the ACM on Programming Languages}, 3\penalty0
  (POPL), 2019.
\newblock \doi{10.1145/3290349}.

\bibitem[Van~Krieken et~al.(2024)Van~Krieken, Badreddine, Manhaeve, and
  Giunchiglia]{vankriekenULLERUnifiedLanguage2024}
E.~Van~Krieken, S.~Badreddine, R.~Manhaeve, and E.~Giunchiglia.
\newblock {{ULLER}}: {{A Unified Language}} for {{Learning}} and {{Reasoning}}.
\newblock In T.~R. Besold, A.~{d'Avila Garcez}, E.~{Jimenez-Ruiz},
  R.~Confalonieri, P.~Madhyastha, and B.~Wagner, editors,
  \emph{Neural-{{Symbolic Learning}} and {{Reasoning}}}, volume 14979, pages
  219--239. Springer Nature Switzerland, Cham, 2024.
\newblock ISBN 978-3-031-71166-4 978-3-031-71167-1.
\newblock \doi{10.1007/978-3-031-71167-1_12}.

\bibitem[van Krieken et~al.(2024)van Krieken, Minervini, Ponti, and
  Vergari]{kriekenIndependenceAssumptionNeurosymbolic2024}
E.~van Krieken, P.~Minervini, E.~M. Ponti, and A.~Vergari.
\newblock On the {{Independence Assumption}} in {{Neurosymbolic Learning}},
  June 2024.

\bibitem[van Krieken et~al.(2025)van Krieken, Minervini, Ponti, and
  Vergari]{kriekenNeurosymbolicReasoningShortcuts2025}
E.~van Krieken, P.~Minervini, E.~Ponti, and A.~Vergari.
\newblock Neurosymbolic {{Reasoning Shortcuts}} under the {{Independence
  Assumption}}, July 2025.

\bibitem[Varnai and Dimarogonas(2020)]{varnaiRobustnessMetricsLearning2020}
P.~Varnai and D.~V. Dimarogonas.
\newblock On {{Robustness Metrics}} for {{Learning STL Tasks}}.
\newblock In \emph{2020 {{American Control Conference}} ({{ACC}})}, pages
  5394--5399, July 2020.
\newblock \doi{10.23919/ACC45564.2020.9147692}.

\bibitem[Walicki and Meldal(1994)]{walicki1994multialgebras}
M.~Walicki and S.~Meldal.
\newblock Multialgebras, power algebras and complete calculi of identities and
  inclusions.
\newblock In \emph{Workshop on the Specification of Abstract Data Types}, pages
  453--468. Springer, 1994.

\end{thebibliography}
\appendix
\section{Brief Introduction to Category Theory}\label{sec:cat-gen}

We recall some basic notions of category theory. See \cite{Spivak2014} for an introduction with a focus on application in database theory, \cite{awodeyCategoryTheory2010} for a logical and type-theoretic overview, and \cite{maclaneCategoriesWorkingMathematician1978} as a general reference.

\begin{definition}
	A Category \category{C} consists of
	\begin{itemize}
		\item a class \categoryobjects{C} of objects,
		\item for any two objects $A, B \in \categoryobjects{C}$ a set $\category{C}(A,B)$ of morphisms from $A$ to $B$.
			$f \in \category{C}(A,B)$ is written as $f: A \to B$ (not necessarily a function),
		\item for any object $A \in \categoryobjects{C}$ an identity morphism $\id_A \in \category{C}(A,A)$, i.e. $\id_A: A \to A$,
		\item for any $A, B, C \in \categoryobjects{C}$ a composition operation
			$\circ : \category{C}(B,C) \times \category{C}(A,B) \to \category{C}(A,C)$,
			i.e. for $f: A \to B, g: B \to C$, we have $g \circ f: A \to C$, 
	\end{itemize}
	such that
	\begin{itemize}
		\item identities are neutral elements for composition, i.e. $f \circ \id_A = f = \id_B \circ f$, and
		\item composition is associative, i.e. $(f \circ g) \circ h = f \circ (g \circ h)$.
	\end{itemize}
\end{definition}

\noindent
Examples of categories are:
	\begin{itemize}
		\item \category{Set}: Sets and functions.
		\begin{itemize}
			\item $\categoryobjects{Set} = \{M \mid M \text{ is a set}\}$
			\item $\category{Set}(A,B) = \{f \mid f: A \to B \text{ is a function}\}$
			\item Compositions and identities of functions.
		\end{itemize}
		\item \category{Meas}: measurable spaces and measurable functions.\footnote{More details at \url{https://ncatlab.org/nlab/show/measurable+space}.}
		\begin{itemize}
			\item $\categoryobjects{Meas} = \{(X, \Sigma_X) \mid \Sigma_X \text{ is a } \sigma\text{-algebra on } X\}$
			\item $\category{Meas}((X, \Sigma_X), (Y, \Sigma_Y)) = \{f: X \to Y \mid f^{-1}(B) \in \Sigma_X \text{ for all } B \in \Sigma_Y\}$
			\item Compositions and identities of measurable functions.
		\end{itemize}
		\item \category{Measr}: measure spaces and measurable functions.\footnote{More details at \url{https://ncatlab.org/nlab/show/measure+space}.}
		\begin{itemize}
			\item $\categoryobjects{Measr} = \{(X, \Sigma_X, \mu_X) \mid \mu_X \text{ is a measure on } (\Sigma_X, X)\}$
			\item $\category{Measr}((X, \Sigma_X, \mu_X), (Y, \Sigma_Y, \mu_Y)) = \{f: X \to Y \mid f \text{ measurable}\}$
			\item Compositions and identities of measurable functions.
		\end{itemize}
		\item \category{Prob}: probability spaces and measurable functions.\footnote{More details at \url{https://ncatlab.org/nlab/show/Prob}.}
		\begin{itemize}
			\item $\categoryobjects{Prob} = \{(X, \Sigma_X, \rho_X) \mid \rho_X \text{ is a probability measure on } (\Sigma_X, X)\}$
			\item $\category{Prob}((X, \Sigma_X, \rho_X), (Y, \Sigma_Y, \rho_Y)) = \{f: X \to Y \mid f \text{ is measurable}\}$
			\item Compositions and identities of measurable functions.
		\end{itemize}
	\end{itemize}

\noindent
The key distinction between $\Meas$ and $\mathbf{Measr}$ is that objects of $\Meas$ are measurable spaces $(X, \Sigma_X)$ carrying only a $\sigma$-algebra, while objects of $\mathbf{Measr}$ are measure spaces $(X, \Sigma_X, \mu_X)$ additionally equipped with a measure. This extra structure is needed when quantifiers require sort-dependent measures (see Table~\ref{tab:prob-cats}). In both categories, morphisms are measurable functions---in particular, morphisms in $\mathbf{Measr}$ are \emph{not} required to preserve the measure.

\begin{definition}\label{def:qbs}
A \textbf{quasi-Borel space} \citep{heunenConvenientCategoryHigherOrder2017} is a set $X$ together with a set $M_X \subseteq (\mathbb{R} \to X)$ of \emph{admissible random elements} satisfying: (1) all constant functions are in $M_X$; (2) $M_X$ is closed under precomposition with measurable functions $\mathbb{R} \to \mathbb{R}$; and (3) $M_X$ is closed under countable case-splitting along measurable partitions of $\mathbb{R}$. A morphism $f : (X, M_X) \to (Y, M_Y)$ is a function $f : X \to Y$ such that $f \circ \alpha \in M_Y$ for all $\alpha \in M_X$.
\end{definition}

\noindent
The category $\mathbf{QBS}$ of quasi-Borel spaces is Cartesian closed and carries a probability monad that, when restricted to standard Borel spaces, agrees with the Giry monad on $\Meas$. This makes $\mathbf{QBS}$ the natural setting for higher-order probabilistic programming and, by extension, for higher-order extensions of NeSyCat; see Section~\ref{subsec:monadic-pp} for further discussion.

Commutative diagrams are often used to visualise equalities between (compositions of) morphisms in categories. For example, the following diagram shows the composition of morphisms $f: A \to B$ and $g: B \to C$:
\begin{center}
\begin{tikzpicture}
				\node (a) at (1,1) {$A$};
				\node (b) at (2,2) {$B$};
				\node (c) at (3,1) {$C$};
				\draw[->] (a) edge node [above left] {\tiny {$f$}} (b);
				\draw[->] (b) edge node [above right] {\tiny {$g$}} (c);
				\draw[->] (a) edge node [above] {\tiny {$g \circ f$}} (c);
\end{tikzpicture}
\end{center}

\begin{definition}
	An object $1 \in \categoryobjects{C}$ is called \emph{terminal}, if for each $A \in \categoryobjects{C}$ there exists a unique morphism $!_A : A \to 1_\category{C}$.
\end{definition}

\begin{definition}[Products]
Let $\mathcal{C}$ be a category and let $A, B$ be objects in $\mathcal{C}$. A \emph{product} of $A$ and $B$ is an object $A \times B$ together with two morphisms (called projections) $\pi_A : A \times B \to A$ and $\pi_B : A \times B \to B$ such that for any object $X$ with morphisms $f : X \to A$ and $g : X \to B$, there exists a unique morphism $u : X \to A \times B$ such that $\pi_A \circ u = f$ and $\pi_B \circ u = g$. This unique $u$ is noted as $\langle f, g \rangle$ and is called the \emph{pairing} of $f$ and $g$. 
This universal property can be depicted by the following commutative diagram:
\[
\begin{tikzcd}
& X \ar[dashed]{d}[swap]{u=\langle f, g \rangle} \ar[bend left]{dr}{g} \ar[bend right]{dl}[swap]{f} \\
A & A \times B \ar{l}[swap]{\pi_A} \ar{r}{\pi_B} & B
\end{tikzcd}
\]
\end{definition}
This easily generalises to products of finitely many objects.
A category having a terminal object and binary products is called \emph{category with finite products}, which can also be considered as a \emph{cartesian monoidal category} up to the choice of the product.\footnote{For more information on cartesian monoidal categories, see \url{https://ncatlab.org/nlab/show/cartesian+monoidal+category}.}

\subsection*{Glossary of Algebraic and Categorical Terms}\label{sec:glossary}

We collect here brief definitions of standard algebraic and categorical notions used in the main text. For further background, see \cite{awodeyCategoryTheory2010,maclaneCategoriesWorkingMathematician1978,hajek_metamathematics_1998,giry1982categorical}. 

\begin{description}[style=nextline, leftmargin=1.5em]

\item[Monoid.]
A \emph{monoid} $(M, \cdot, e)$ is a set $M$ equipped with an associative binary operation $\cdot : M \times M \to M$ and a neutral element $e \in M$ satisfying $e \cdot x = x = x \cdot e$ for all $x \in M$. The monoids $(S, \otimes, 1)$ and $(S, \oplus, 0)$ in the 2Mon-BLat (Def.~\ref{def:2Mon-BLat}) model conjunction and disjunction, respectively.

\item[Bounded lattice.]
A \emph{bounded lattice} $(S, \leq, \bot, \top)$ is a partially ordered set in which every pair of elements $a, b \in S$ has a greatest lower bound (meet) $a \wedge b$ and a least upper bound (join) $a \vee b$, and in which there exist elements $\bot$ (bottom) and $\top$ (top) satisfying $\bot \leq x \leq \top$ for all $x \in S$.

\item[$\sigma$-algebra and measurable space.]
A \emph{$\sigma$-algebra} on a set $X$ is a collection $\Sigma_X \subseteq \mathcal{P}(X)$ containing $X$, closed under complements and countable unions. The pair $(X, \Sigma_X)$ is called a \emph{measurable space}. A function $f : X \to Y$ between measurable spaces is \emph{measurable} if $f^{-1}(B) \in \Sigma_X$ for every $B \in \Sigma_Y$.

\item[Measure and probability measure.]
A \emph{measure} on a measurable space $(X, \Sigma_X)$ is a countably additive function $\mu : \Sigma_X \to [0,\infty]$ with $\mu(\emptyset) = 0$. If $\mu(X) = 1$, then $\mu$ is a \emph{probability measure} and $(X, \Sigma_X, \mu)$ is a \emph{probability space}.

\item[Markov kernel.]
A \emph{Markov kernel} from a measurable space $(X, \Sigma_X)$ to $(Y, \Sigma_Y)$ is a function $k : X \times \Sigma_Y \to [0,1]$ such that $k(x, \cdot)$ is a probability measure on $Y$ for each $x \in X$, and $k(\cdot, B)$ is a measurable function on $X$ for each $B \in \Sigma_Y$. Equivalently, it is a measurable map $X \to \mathcal{G}(Y)$, i.e.\ a Kleisli morphism for the Giry monad.

\item[Functor.]
A \emph{functor} $F : \mathcal{C} \to \mathcal{D}$ between categories maps objects to objects and morphisms to morphisms, preserving identities ($F(\id_A) = \id_{F(A)}$) and composition ($F(g \circ f) = F(g) \circ F(f)$).

\item[Natural transformation.]
Given functors $F, G : \mathcal{C} \to \mathcal{D}$, a \emph{natural transformation} $\alpha : F \Rightarrow G$ is a family of morphisms $\alpha_A : F(A) \to G(A)$ indexed by objects $A$ of $\mathcal{C}$, such that $\alpha_B \circ F(f) = G(f) \circ \alpha_A$ for every morphism $f : A \to B$.

\item[Kleisli category.]
For a monad $({\mathcal T}, \eta, (-)^*)$ on a category $\mathcal{C}$, the \emph{Kleisli category} $\Kl{{\mathcal T}}$ has the same objects as $\mathcal{C}$, but a morphism $A \to B$ in $\Kl{{\mathcal T}}$ is a morphism $A \to {\mathcal T} B$ in $\mathcal{C}$. Composition is given by Kleisli composition: $g \circ_\mathrm{Kl} f := g^* \circ f$.

\item[Adjunction and right adjoint.]
An \emph{adjunction} $F \dashv G$ between categories $\mathcal{C}$ and $\mathcal{D}$ consists of functors $F : \mathcal{C} \to \mathcal{D}$ (the \emph{left adjoint}) and $G : \mathcal{D} \to \mathcal{C}$ (the \emph{right adjoint}) together with a natural bijection $\mathcal{D}(F(A), B) \cong \mathcal{C}(A, G(B))$. In the 2Mon-BLat (Def.~\ref{def:2Mon-BLat}), when $\to$ is the right adjoint of $\otimes$, this means $a \otimes b \leq c$ if and only if $a \leq b \to c$ (the \emph{residuation} or \emph{deduction} property).

\item[Cartesian closed category.]
A category is \emph{Cartesian closed} if it has finite products and, for every object $B$, the functor $(-) \times B$ has a right adjoint $(-)^B$ (the \emph{exponential} or \emph{internal hom}). This provides currying: morphisms $A \times B \to C$ correspond bijectively to morphisms $A \to C^B$. The category $\Set$ is Cartesian closed; $\Meas$ is not (which motivates $\mathbf{QBS}$, Def.~\ref{def:qbs}).

\item[t-norm and t-conorm.]
A \emph{t-norm} $\otimes : [0,1]^2 \to [0,1]$ is a commutative, associative, monotone operation with neutral element $1$. A \emph{t-conorm} $\oplus : [0,1]^2 \to [0,1]$ is dual: commutative, associative, monotone with neutral element $0$. These provide the algebraic operations for fuzzy and probabilistic logics in the 2Mon-BLat framework (Sec.~\ref{sec:algebraic-overview}).

\end{description}

\end{document}